\documentclass[a4paper]{article}
\usepackage{microtype}
\usepackage{graphicx}
\usepackage{subfigure}
\usepackage{booktabs} 

\usepackage{amssymb}
\usepackage{graphicx}
\usepackage{amsmath}
\usepackage{xspace}
\usepackage{bm}
\usepackage{amsthm}

\usepackage[numbers,sort&compress]{natbib}
\usepackage{color}
\usepackage{mathtools}
\usepackage{xr}

\usepackage{microtype}
\usepackage{graphicx}
\usepackage{subfigure}
\usepackage{booktabs} 

\usepackage{hyperref}

\usepackage{multirow}
\usepackage{comment}
\usepackage{afterpage}
\usepackage{indentfirst}
\usepackage{enumerate}

\usepackage{amsmath}
\usepackage{amssymb}
\usepackage{mathtools}
\usepackage{amsthm}

\usepackage[capitalize,noabbrev]{cleveref}

\theoremstyle{plain}
\newtheorem{theorem}{Theorem}[section]
\newtheorem{proposition}[theorem]{Proposition}
\newtheorem{lemma}[theorem]{Lemma}
\newtheorem{corollary}[theorem]{Corollary}
\theoremstyle{definition}
\newtheorem{definition}[theorem]{Definition}

\theoremstyle{remark}
\newtheorem{remark}[theorem]{Remark}

\newcommand{\nbb}{\mathbb{N}}
\newcommand{\bz}{\mathbf{z}}

\newcommand{\bw}{\mathbf{w}}
\newcommand{\fcal}{\mathcal{F}}
\newcommand{\acal}{\mathcal{A}}
\newcommand{\ibb}{\mathbb{I}}
\newcommand{\xcal}{\mathcal{X}}

\newcommand{\ucal}{\mathcal{U}}
\newcommand{\vcal}{\mathcal{V}}

\newcommand{\ccal}{\mathcal{C}}
\newcommand{\hcal}{\mathcal{H}}
\newcommand{\bx}{\mathbf{x}}
\newcommand{\by}{\mathbf{y}}

\newcommand{\bu}{\mathbf{u}}

\newcommand{\zcal}{\mathcal{Z}}
\newcommand{\ncal}{\mathcal{N}}
\newcommand{\tcal}{\mathcal{T}}

\newcommand{\dcal}{\mathcal{D}}

\newcommand{\gcal}{\mathcal{G}}

\newcommand{\ebb}{\mathbb{E}}

\newcommand{\kk}{K}

\newcommand{\bv}{\mathbf{v}}
\newcommand{\fat}{\mathrm{fat}}
\newcommand{\rbb}{\mathbb{R}}

\numberwithin{equation}{section}

\usepackage[textsize=tiny]{todonotes}

\title{Generalization Analysis for Contrastive Representation Learning}
\author{%
  Yunwen Lei$^{1}$\quad Tianbao Yang$^{2}$\quad Yiming Ying$^{3}$\footnote{Corresponding author}\quad Ding-Xuan Zhou$^{4}$\\[1.2pt]
  $^1$Department of Mathematics, Hong Kong Baptist University\\[1.2pt]
  $^2$Department of Computer Science and Engineering, Texas A\&M University\\[1.2pt]
  $^3$Department of Mathematics and Statistics, State University\\ of New York at Albany\\[1.2pt]
  $^4$School of Mathematics and Statistics, University of Sydney\\[1.2pt]
  \texttt{yunwen@hkbu.edu.hk} \quad \texttt{tianbao-yang@tamu.edu} \quad \texttt{yying@albany.edu} \\[1pt] \texttt{dingxuan.zhou@sydney.edu.au}
}

\begin{document}

\maketitle

\linespread{1.3}

\begin{abstract}

Recently, contrastive learning has found impressive success in advancing the state of the art in solving various machine learning tasks. However, the existing generalization analysis is very limited or even not meaningful. In particular, the existing generalization error bounds depend linearly on the number $k$ of negative examples while it was widely shown in practice that choosing a large $k$ is necessary to guarantee good generalization of contrastive learning in downstream tasks.  In this paper, we establish novel generalization bounds for contrastive learning which do not depend on  $k$, up to logarithmic terms. Our analysis uses structural results on empirical covering numbers and Rademacher complexities to exploit the Lipschitz continuity of loss functions. For self-bounding Lipschitz loss functions, we further improve our results by developing optimistic bounds which imply fast rates in a low noise condition. We apply our results to learning with both linear representation and nonlinear representation by deep neural networks, for both of which we derive Rademacher complexity bounds to get improved generalization bounds.
\end{abstract}

\section{Introduction}

The performance of machine learning (ML) models often depends largely on the representation of data, which motivates a resurgence of contrastive representation learning (CRL) to learn a representation function $f:\xcal\mapsto\rbb^d$ from unsupervised data~\citep{chen2020simple,khosla2020supervised,he2020momentum}.
The basic idea is to pull together similar pairs $(\bx,\bx^+)$ and push apart disimilar pairs $(\bx,\bx^-)$ in an embedding space, which can be formulated as minimizing the following objective~\citep{chen2020simple,oord2018representation}
\[
\!\ebb_{\bx,\bx^+,\{\bx_i^{-}\}_{i=1}^k}\!\log\!\Big(1+\sum_{i=1}^k\exp\Big(-f(\bx)^\top\big(f(\bx^+)-f(\bx_i^-)\big)\!\Big)\!\Big),
\]
where $k$ is the number of negative examples.
The hope is that the learned representation $f(\bx)$ would capture the latent structure and be beneficial to other downstream learning tasks~\citep{arora2019theoretical,tosh2021contrastive}. 
CRL has achieved impressive empirical performance in advancing the state-of-the-art performance in various domains such as computer vision~\citep{he2020momentum,caron2020unsupervised,chen2020simple,caron2020unsupervised} and natural language processing~\citep{brown2020language,gao2021simcse,radford2021learning}.

The empirical success of CRL motivates a natural question on theoretically understanding how the learned representation adapts to the downstream tasks, i.e.,

\emph{How would the generalization behavior of downstream ML models benefit from the representation function built from positive and negative pairs? Especially, how would the number of negative examples affect the learning performance?}

\citet{arora2019theoretical} provided an attempt to answer the above questions by developing a theoretical framework to study CRL. They first gave generalization bounds for a learned representation function in terms of Rademacher complexities. Then, they showed that this generalization behavior measured by an unsupervised loss guarantees the generalization behavior of a linear classifier in the downstream classification task.  However, the generalization bounds there enjoy a linear dependency on $k$, which would not be effective if $k$ is large. Moreover, this is not consistent with many studies which show a large number of negative examples~\citep{chen2020simple,tian2020contrastive,henaff2020data,khosla2020supervised} is necessary for good generalization performance. For example, the work~\citep{he2020momentum} used $65536$ negative examples in unsupervised visual representation learning, for which the existing analysis requires $n\geq (65536)^2d$ training examples to get non-vacuous bounds~\citep{arora2019theoretical}. Therefore, the existing analysis does not fully answer the question on the super performance of CRL to downstream tasks as already shown in many applications. \citet{sogclr} has demonstrated the benefits using all negative data for each anchor data for CRL and proposed an efficient algorithms for optimizing global contrastive loss. Therefore, the existing analysis does not fully answer the question on the super performance of CRL to downstream tasks as already shown in many applications.

\begin{table}[tbp]
\centering\renewcommand{\arraystretch}{1.3}
  \begin{tabular}{|c|c|c|}
    \hline
     Assumption & Arora et al.'19 & Ours\\ \hline
     & $O\Big({k \sqrt{d} \over \sqrt{n} }\Big)$ & $ \widetilde{O}\Big({\sqrt{d}\over \sqrt{n}}\Big)$ \\ \hline
     low noise & $O\Big({ k \sqrt{d} \over \sqrt{n} }\Big)$ & $ \widetilde{O}\Big({d\over {n}}\Big)$ \\ \hline
  \end{tabular}
  \caption{Comparison between our generalization bounds and those in \citet{arora2019theoretical} for the logistic loss. Here $d$ is the number of learned features. The notation $\widetilde{O}$ ignores log factors.}
\end{table}

In this paper, we aim to further deepen our understanding of CRL by fully exploiting the Lipschitz continuity of loss functions. Our contributions are listed as follows.

\noindent 1. We develop generalization error bounds for CRL. We consider three types of loss functions: $\ell_2$-Lipschitz loss, $\ell_\infty$-Lipschitz loss and self-bounding Lipschitz loss. For $\ell_2$-Lipschitz loss, we develop a generalization bound with a square-root dependency on $k$ by two applications of vector-contraction lemmas on Rademacher complexities, which improves the existing bound by a factor of $\sqrt{k}$~\citep{arora2019theoretical}. For $\ell_\infty$-Lipschitz loss, we develop generalization bounds which does not depend on $k$, up to some logarithmic terms, by approximating the arguments of loss functions via expanding the original dataset by a factor of $k$ to fully exploit the Lipschitz continuity.  For self-bounding Lipschitz loss, we develop optimistic bounds involving the training errors, which can imply fast rates under a low noise setting. All of our generalization bounds involve Rademacher complexities of feature classes, which preserve the coupling among different features.

\noindent 2. We then apply our general result to two unsupervised representation learning problems: learning with linear features and learning with nonlinear features via deep neural networks (DNNs). For learning with linear features, we consider two regularization schemes, i.e., $p$-norm regularizer and Schatten-norm regularizer. For learning with nonlinear features, we develop Rademacher complexity and generalization bounds with a square-root dependency on the depth of DNNs. To this aim, we adapt the technique in \citet{golowich2018size} by using a different moment generalization function to capture the coupling among different features.

\noindent 3. Finally, we apply our results on representation learning to the generalization analysis of downstream classification problems, which outperforms the existing results by a factor of $k$ (ignoring a log factor).

The remaining parts of the paper are organized as follows. Section \ref{sec:work} reviews the related work, and Section \ref{sec:problem} provides the problem formulation. We give generalization bounds for CRL in Section \ref{sec:gen-error} for three types of loss functions, which are then applied to learning with both linear and nonlinear features in Section \ref{sec:app}. Conclusions are given in Section \ref{sec:conclusion}.

\section{Related Work\label{sec:work}}

The most related work is the generalization analysis of CRL in \citet{arora2019theoretical}, where the authors developed generalization bounds for unsupervised errors in terms of Rademacher complexity of representation function classes. Based on this, they further studied the performance of linear classifiers on the learned features. In particular, they considered the mean classifier where the weight for a class label is the mean of the representation of corresponding inputs. A major result in \citet{arora2019theoretical} is to show that the classification errors of the mean classifier can be bounded by the unsupervised errors of learned representation functions. This shows that the downstream classification task can benefit from a learned representation function with a low unsupervised error.

The above work motivates several interesting theoretical study of CRL. \citet{nozawa2020pac} studied CRL in a PAC-Bayesian setting, which aims to learn a posterior distribution of representation functions. \citet{nozawa2020pac} derived PAC-Bayesian bounds for the posterior distribution and applied it to get PAC-Bayesian bounds for the mean-classifier, which relaxes the i.i.d. assumption. Negative examples in the framework \citep{arora2019theoretical} are typically taken to be randomly sampled datapoints, which may actually have the same label of the point of interest. This introduces a bias in the objective function of CRL, which leads to performance drops in practice. Motivated by this, \citet{chuang2020debiased} introduced a debiased CRL algorithm by building an approximation of unbiased error in CRL, and developed generalization guarantees for the downstream classification.

Several researchers studied CRL from other perspectives. \citet{lee2021predicting} proposed to learn a representation function $f$ to minimize $\ebb_{(X_1,X_2)}[\|X_2-f(X_1)\|_2^2]$, where $X_1,X_2$ are unlabeled input and pretext target. Under an approximate conditional independency assumption, the authors showed that a linear function based on the learned representation approximates the ground true predictor on downstream problems. In a generative modeling setup, \citet{tosh2021contrastiveb} proposed to learn representation functions by a landmark embedding procedure, which can reveal the underlying topic posterior information. \citet{tosh2021contrastive} studied CRL in a multi-view setting with two views available for each datum. Under an assumption on the redundancy between the two views, the authors showed that low-dimensional representation can achieve near optimal downstream performance with linear models.
\citet{haochen2021provable} studied self-supervised learning from the perspective of spectral clustering based on a population augmentation graph, and proposed a spectral contrastive loss. They further developed generalization bounds for both representation learning and the downstream classification. This result is improved in a recent work by developing a guarantee to incorporate the representation function class~\citep{pmlr-v162-saunshi22a}.  There are also recent work on theoretical analysis of representation learning via gradient-descent dynamics~\citep{lee2021predicting,tian2020understanding}, mutual information~\citep{tsai2020demystifying}, alignment of representations~\citep{wang2020understanding}, and causality~\citep{mitrovic2020representation}. CRL is related to metric learning, for which generalization bounds have been studied in the literature~\citep{cao2012generalization}.

\section{Problem Formulation\label{sec:problem}}

Let $\xcal$ denote the space of all possible datapoints.
In CRL, we are given several \emph{similar} data in the form of pairs $(\bx,\bx^+)$ drawn from a distribution $\dcal_{sim}$ and \emph{negative} data $\bx_1^-,\bx_2^-,\ldots,\bx_k^-$ drawn from a distribution $\dcal_{neg}$ unrelated to $\bx$. Our aim is to learn a feature map $f:\xcal\mapsto\rbb^d$ from a class of representation functions $\fcal=\big\{f:\|f(\cdot)\|_2\leq R\big\}$ for some $R>0$, where $\|\cdot\|_2$ denotes the Euclidean norm. Here $d\in\nbb$ denotes the number of features.

We follow the framework in \citet{arora2019theoretical} to define the distribution $\dcal_{sim}$ and the distribution $\dcal_{neg}$. Let $\ccal$ denote the set of all latent classes and for each class $c\in\ccal$ we assume there is a probability distribution $\dcal_c$ over $\xcal$, which quantifies the relevance of $\bx$ to the class $c$. We assume there is a probability distribution $\rho$ defined over $\ccal$. Then we define $\dcal_{sim}(\bx,\bx^+)$ and $\dcal_{neg}(\bx^-)$ as follows
\begin{gather*}
  \dcal_{sim}(\bx,\bx^+)=\ebb_{c\sim\rho}\big[\dcal_c(\bx)\dcal_c(\bx^+)\big], \\
  \dcal_{neg}(\bx^-)=\ebb_{c\sim\rho}\big[\dcal_c(\bx^-)\big].
\end{gather*}
Intuitively, $\dcal_{sim}(\bx,\bx^+)$ measures the probability of $\bx$ and $\bx^+$ being drawn from the same class $c\sim\rho$, while $\dcal_{neg}(\bx^-)$ measures the probability of drawing an un-relevant $\bx^-$.
Let $(\bx_j,\bx_j^+)\sim \dcal_{sim}$ and $(\bx_{j1}^-,\ldots,\bx_{jk}^-)\sim\dcal_{neg},j\in[n]:=\{1,\ldots,n\}$, where $k$ denotes the number of negative examples. We collect these training examples into a dataset
\begin{equation}\label{S}
S=\Big\{(\bx_1,\bx_1^+,\bx_{11}^-,\ldots,\bx_{1k}^-),
(\bx_2,\bx_2^+,\bx_{21}^-,\ldots,\bx_{2k}^-),
\ldots, (\bx_n,\bx_n^+,\bx_{n1}^-,\ldots,\bx_{nk}^-)
\Big\}.
\end{equation}

Given a representation function $f$, we can measure its performance by building a classifier based on this representation and computing the accuracy of the classifier. To this aim, we define a $(\kk+1)$-way supervised task $\tcal$ consisting of distinct classes $\{c_1,\ldots,c_{\kk+1}\}\subseteq \ccal$. The examples for this supervised task are drawn by the following process:

We first draw a label $c\in\tcal=\{c_1,\ldots,c_{\kk+1}\}$ from a distribution $\dcal_{\tcal}$ over $\tcal$, after which we draw an example $\bx$ from $\dcal_c$. This defines the following distribution over labeled pairs $(\bx,c)$:
$
\dcal_{\tcal}(\bx,c)=\dcal_c(\bx)\dcal_{\tcal}(c).
$
Since there is a label for each example, we can build a multi-class classifier $g:\xcal\mapsto\rbb^{\kk+1}$ for $\tcal$, where $g_c(\bx)$ measures the ``likelihood'' of assigning the class label $c$ to the example $\bx$. The loss of $g$ on a point $(\bx,y)\in\xcal\times\tcal$ can be measured by $\ell_s\big(\big\{g(\bx)_y-g(\bx)_{y'}\}_{y'\neq y}\big)$, where $\ell_s:\rbb^{\kk}\mapsto\rbb^+$.
We quantify the performance of a classifier $g$ on the task $\tcal$ by the supervised loss. By minimizing the supervised loss, we want to build a classifier whose component associated to the correct label is largest.
\begin{definition}[Supervised loss]
  Let $g:\xcal\mapsto\rbb^{\kk+1}$ be a multi-class classifier. The supervised loss of $g$ is defined as
  \[
  L_{sup}(\tcal,g):=\ebb_{(\bx,c)\sim\dcal_{\tcal}}\Big[\ell_s\big(\big\{g(\bx)_c-g(\bx)_{c'}\big\}_{c'\neq c}\big)\Big].
  \]
\end{definition}
For CRL, we often consider $g$ as a linear classifier based on the learned representation $f$, i.e., $g(\bx)=Wf(\bx)$, where $W\in\rbb^{(\kk+1)\times d}$. Then the performance of the representation function $f(\bx)$ can be quantified by the accuracy of the best linear classifier on the representation $f(\bx)$:
\[
L_{sup}(\tcal,f)=\min_{W\in\rbb^{(\kk+1)\times d}}L_{sup}(\tcal,Wf).
\]
%

To find a good representation $f$ based on unsupervised dataset $S$, we need to introduce the concept of unsupervised loss functions. Let $\ell:\rbb^k\mapsto\rbb_+$ be a loss function for which popular choices include the hinge loss
\begin{equation}\label{loss-a}
  \ell(\bv)=\max\big\{0,1+\max_{i\in[k]}\{-v_i\}\big\}
\end{equation}
and the logistic loss
\begin{equation}\label{loss-b}
  \ell(\bv)=\log\big(1+\sum_{i\in[k]}\exp(-v_i)\big).
\end{equation}
Let $f(\bx)^\top$ denote the transpose of $f(\bx)$.
\begin{definition}[Unsupervised error]
The population unsupervised error is defined as
\[
L_{un}(f):=\ebb\big[\ell\big(\big\{f(\bx)^\top(f(\bx^+)-f(\bx_i^-))\big\}_{i=1}^k\big)\big].
\]
The empirical unsupervised error with $S$ is defined as
\[
\hat{L}_{un}(f):=\frac{1}{n}\sum_{j=1}^{n}\ell\big(\big\{f(\bx_j)^\top(f(\bx_j^+)-f(\bx_{ji}^-))\big\}_{i=1}^k\big).
\]
\end{definition}

A natural algorithm is to find among $\fcal$ the function with the minimal empirical unsupervised loss, i.e., $\hat{f}:=\arg\min_{f\in\fcal}\hat{L}_{un}(f)$. This function can then be used for the downstream supervised learning task, e.g., to find a linear classifier $g(\bx)=Wf(\bx)$ indexed by $W\in\rbb^{(\kk+1)\times d}$.

\section{Generalization Error Bounds\label{sec:gen-error}}
In this paper, we are interested in the performance of $\hat{f}$ on testing, i.e., how the empirical behavior of $\hat{f}$ on $S$ would generalize well to testing examples. Specifically, we will control $L_{un}(\hat{f})-\hat{L}_{un}(\hat{f})$. Since $\hat{f}$ depends on the dataset $S$, we need to control the uniform deviation between population unsupervised error and empirical unsupervised error over the function class $\fcal$, which depends on the complexity of $\fcal$. In this paper, we will use Rademacher complexity to quantify the complexity of $\fcal$~\citep{bartlett2002rademacher}.

\begin{definition}[Rademacher Complexity]\label{def:rademacher}
  Let $\widetilde{\fcal}$ be a class of real-valued functions over a space $\zcal$ and $\widetilde{S}=\{\bz_i\}_{i=1}^n\subseteq\zcal$.
  The \emph{empirical} Rademacher complexity of $\widetilde{\fcal}$ with respect to (w.r.t.) $\widetilde{S}$ is defined as
  $
    \mathfrak{R}_{\widetilde{S}}(\widetilde{\fcal})=\ebb_{\bm{\epsilon}}\big[\sup_{f\in \widetilde{\fcal}}\frac{1}{n}\sum_{i\in[n]}\epsilon_if(\bz_i)\big],
  $
  where $\bm{\epsilon}=(\epsilon_i)_{i\in[n]}\sim\{\pm1\}^n$ are independent Rademacher variables.
  We define the \emph{worst-case} Rademacher complexity as
  $
  \frak{R}_{\zcal,n}(\widetilde{\fcal})=\sup_{\widetilde{S}\subseteq\zcal:|\widetilde{S}|=n}\frak{R}_{\widetilde{S}}(\widetilde{\fcal}),
  $
  where $|\widetilde{S}|$ is the cardinality of~$\widetilde{S}$.
\end{definition}

For any $f\in\fcal$, we introduce $g_f:\xcal^{k+2}\mapsto\rbb$ as follows
\[
g_f(\bx,\bx^+,\bx_1^-,\ldots,\bx_k^-)=\ell\big(\big\{f(\bx)^\top\big(f(\bx^+)-f(\bx_i^-)\big)\big\}_{i=1}^k\big).
\]
It is then clear that
\[
L_{un}(f)\!-\!\hat{L}_{un}(f)=\ebb_{\bx,\bx^+,\bx_1^-,\ldots,\bx_k^-}\big[g_f(\bx,\bx^+,\bx_1^-,\ldots,\bx_k^-)\big]
-\frac{1}{n}\sum_{j\in[n]}g_f(\bx_j,\bx_j^+,\bx_{j1}^-,\ldots,\bx_{jk}^-).
\]
Results in learning theory show that we can bound $L_{un}(\hat{f})-\hat{L}_{un}(\hat{f})$ by $\mathfrak{R}_S(\gcal)$~\citep{bartlett2002rademacher},~where
\[
\gcal=\big\{(\bx,\bx^+,\bx_1^-,\ldots,\bx_k^-)\mapsto
 g_f(\bx,\bx^+,\bx_1^-,\ldots,\bx_k^-):f\in\fcal\big\}.
\]
Note functions in $\gcal$ involve the nonlinear function $\ell:\rbb^k\mapsto\rbb_+$, which introduces difficulties in the complexity analysis. Our key idea is to use the Lipschitz continuity of $\ell$ to reduce the complexity of $\gcal$ to the complexity of another function class without $\ell$. Since the arguments in $\ell$ are vectors, we can have different definition of Lipschitz continuity w.r.t. different norms~\citep{lei2015multi,tewari2015generalization,lei2019data,foster2019ell}. For any $\mathbf{a}=(a_1,\ldots,a_k)\in\rbb^k$ and $p\geq1$, we define the $\ell_p$-norm as $\|\mathbf{a}\|_p=\big(\sum_{i=1}^{n}|a_i|^p\big)^{\frac{1}{p}}$.
\begin{definition}[Lipschitz continuity]
We say $\ell:\rbb^k\mapsto\rbb_+$ is $G$-Lipschitz w.r.t. the $\ell_p$-norm iff
\[
|\ell(\mathbf{a})-\ell(\mathbf{a}')|\leq G\|\mathbf{a}-\mathbf{a}'\|_p,\quad\forall\mathbf{a},\mathbf{a}'\in\rbb^k. 
\]
\end{definition}
In this paper, we are particularly interested in the Lipschitz continuity w.r.t. either the $\ell_2$-norm or the $\ell_\infty$-norm. According to Proposition \ref{prop:loss-lip}, the loss functions defined in Eq. \eqref{loss-a} and Eq. \eqref{loss-b} are $1$-Lipschitz continuous w.r.t. $\|\cdot\|_\infty$, and $1$-Lipschitz continuous w.r.t. $\|\cdot\|_2$~\citep{lei2019data}. 

Note each component of the arguments in $\ell$ are of the form $f(\bx)^\top(f(\bx^+)-f(\bx^-))$. This motivates the definition of the following function class
\[
\hcal=\Big\{h_f(\bx,\bx^+,\bx^-)=f(\bx)^\top\big(f(\bx^+)-f(\bx^-)\big):f\in\fcal\Big\}.
\]
As we will see in the analysis, the complexity of $\gcal$ is closely related to that of $\hcal$. Therefore, we first show how to control the complexity of $\hcal$.
In the following lemma, we provide Rademacher complexity bounds of $\hcal$ w.r.t. a general dataset $S'$ of cardinality $n$. We will use a vector-contraction lemma to prove it~\citep{maurer2016vector}. The basic idea is to notice the Lipschitz continuity of the map $(\bx,\bx^+,\bx^-)\mapsto \bx^\top(\bx^+-\bx^-)$ w.r.t. $\|\cdot\|_2$ on $\xcal^3$. The proof is given in Section \ref{sec:proof-rademacher-l2}.
\begin{lemma}\label{lem:rademacher-H}
Let $n\in\nbb$ and $S'=\{(\bx_j,\bx_j^+,\bx_j^-):j\in[n]\}$. Assume $\|f(\bx)\|_2\leq R$ for any $f\in\fcal$ and $\bx\in S'$. Then
  \[
  \mathfrak{R}_{S'}(\hcal)\leq \frac{\sqrt{12}R}{n}\ebb_{\bm{\epsilon}\sim\{\pm1\}^n\times\{\pm1\}^d\times\{\pm1\}^3}\\
  \Big[\sup_{f\in\fcal}\sum_{j\in[n]}\sum_{t\in[d]}\Big(\epsilon_{j,t,1}f_t(\bx_j)+
  \epsilon_{j,t,2}f_t(\bx_j^+)+\epsilon_{j,t,3}f_t(\bx_j^-)\Big)\Big],
  \]
  where $f_t(\bx)$ is the $t$-th component of $f(\bx)\in\rbb^d$.
\end{lemma}
\begin{remark}
  We compare Lemma \ref{lem:rademacher-H} with the following Rademacher complexity bound in \citet{haochen2021provable} 
  \begin{equation}\label{ma}
  \ebb_{\bf{\epsilon}\sim\{\pm\}^n}\Big[\sup_{f\in\fcal}\sum_{j\in[n]}\epsilon_jf(\bx_j)^\top f(\bx_j^+)\Big]\leq \\ d\max_{t\in[d]}\ebb_{\bf{\epsilon}\sim\{\pm\}^n}\Big[\sup_{f_t\in\fcal_t}\sum_{j\in[n]}\epsilon_jf_t(\bx_j)f_t(\bx_j^+)\Big],
  \end{equation}
  where
  $
  \fcal_t=\big\{\bx\mapsto f_t(\bx):f\in\fcal\big\}.
  $
  As a comparison, our analysis in Lemma \ref{lem:rademacher-H} can imply the following bound
  \[
  \ebb_{\bf{\epsilon}\sim\{\pm\}^n}\Big[\sup_{f\in\fcal}\sum_{j\in[n]}\epsilon_jf(\bx_j)^\top f(\bx_j^+)\Big]\leq
  2R\ebb_{\bf{\epsilon}\sim\{\pm\}^{2nd}}\Big[\sup_{f\in\fcal}\sum_{j\in[n]}\sum_{t\in[d]}\big(\epsilon_{j,t,1}f_t(\bx_j)+\epsilon_{j,t,2}f_t(\bx_j^+)\big)\Big].
  \]
  Eq. \eqref{ma} decouples the relationship among different features since the maximization over $t\in[d]$ is outside of the expectation operator. As a comparison, our result preserves this coupling since the summation over $t\in[d]$ is inside the supermum over $f\in\fcal$. This preservation of coupling has an effect on the bound. Indeed, it is expected that 
  \begin{multline*}
  \ebb_{\bf{\epsilon}\sim\{\pm\}^{2nd}}\Big[\sup_{f\in\fcal}\sum_{j\in[n]}\sum_{t\in[d]}\big(\epsilon_{j,t,1}f_t(\bx_j)+\epsilon_{j,t,2}f_t(\bx_j^+)\big)\Big]=\\
  O\Big(\sqrt{d}\ebb_{\bf{\epsilon}\sim\{\pm\}^{2n}}\Big[\sup_{f_t\in\fcal_t}\sum_{j\in[n]}\big(\epsilon_{j,1}f_t(\bx_j)+\epsilon_{j,2}f_t(\bx_j^+)\big)\Big]\Big).
  \end{multline*}
  In this case, our result implies a bound with a better dependency on $d$ as compared to Eq. \eqref{ma} (the factor of $d$ in Eq. \eqref{ma} is replaced by $\sqrt{d}$ here). We can plug our bound into the analysis in \citet{haochen2021provable} to improve their results.
\end{remark}
\begin{remark}
  Lemma \ref{lem:rademacher-H} requires an assumption $\|f(\bx)\|_2\leq R$. This assumption can be achieved by adding a projection operator as $f(\bx)=\mathcal{P}_R(\tilde{f}(\bx))$ for $\tilde{f}\in\fcal$, where $\mathcal{P}_R$ denotes the projection operator onto the Euclidean ball with radius $R$ around the zero point. According to the inequality $\|\mathcal{P}_R(\tilde{f}(\bx))-\mathcal{P}_R(\tilde{f}'(\bx))\|_2\leq \|\tilde{f}(\bx)-\tilde{f}'(\bx)\|_2$, the arguments in the proof indeed show the following inequality with $\hcal=\big\{h_{\tilde{f}}(\bx,\bx^+,\bx^-)=\mathcal{P}_R(\tilde{f}(\bx))^\top\big(\mathcal{P}_R(\tilde{f}(\bx^+))-\mathcal{P}_R(\tilde{f}(\bx^-))\big):\tilde{f}\in\fcal\big\}$:
  \[
  \mathfrak{R}_{S'}(\hcal)\leq \frac{\sqrt{12}R}{n}\ebb_{\bm{\epsilon}\sim\{\pm1\}^{3nd}}
  \Big[\sup_{\tilde{f}\in\fcal}\sum_{j\in[n]}\sum_{t\in[d]}
  \Big(\epsilon_{j,t,1}\tilde{f}_t(\bx_j)+
  \epsilon_{j,t,2}\tilde{f}_t(\bx_j^+)+\epsilon_{j,t,3}\tilde{f}_t(\bx_j^-)\Big)\Big].
  \]
  That is, we can add a projection operator over $\fcal$ to remove the assumption $\|f(\bx)\|_2\leq R$.
\end{remark}

\begin{table}[htbp]
\centering\renewcommand{\arraystretch}{1.3}\setlength{\tabcolsep}{3pt}
  \begin{tabular}{|c|c|c|}
    \hline
    Loss & Arora et al.'19 & Ours\\ \hline
    $1$-$\ell_2$-Lipschitz & $\frac{\sqrt{k}\mathfrak{B}}{n}$ & $\frac{\mathfrak{A}}{n}$ \\ \hline
    $1$-$\ell_\infty$-Lipschitz & $\frac{\sqrt{k}\mathfrak{B}}{n}$ & $\frac{\mathfrak{C}}{n\sqrt{k}} ^*$ \\ \hline
    S.B. $1$-Lipschitz &  \rule[3pt]{36pt}{0.5pt}  & $n^{-1}\!+\!n^{-2}k^{-1}\mathfrak{C}^2\!+\!(n^{-\frac{1}{2}}+n^{-1}k^{-\frac{1}{2}}\mathfrak{C})\hat{L}^{\frac{1}{2}}_{un}(f) ^*$\\
    \hline
  \end{tabular}
  \caption{Comparison between our generalization bounds and those in \citet{arora2019theoretical}. The notation $^*$ means we ignore log factors. S.B. means self-bounding.   The notations $\mathfrak{A}, \mathfrak{B}$ and $\mathfrak{C}$ are defined in Eq. \eqref{aa}, \eqref{bb} and \eqref{cc}, which are typically of the same order. Then, our results improve the bounds in \citet{arora2019theoretical} by a factor of $\sqrt{k}$ for $\ell_2$-Lipschitz loss, and by a factor of $k$ for $\ell_\infty$-Lipscthiz loss. For self-bounding loss, we get optimistic bounds.}
\end{table}




\subsection{$\ell_2$ Lipschitz Loss}
We first consider $\ell_2$ Lipschitz loss.
The following theorem to be proved in Section \ref{sec:proof-rademacher-l2} gives Rademacher complexity and generalization error bounds for unsupervised loss function classes. We always assume $\ell\big(\big\{f(\bx)^\top(f(\bx^+)-f(\bx_i^-))\big\}_{i=1}^k\big)\leq B$ for any $f\in\fcal$ in this paper.
\begin{theorem}[Generalization bound: $\ell_2$-Lipschitz loss\label{thm:rademacher-l2}]
Assume $\|f(\bx)\|_2\leq R$ for any $f\in\fcal$ and $\bx\in\xcal$. Let $S$ be defined as in Eq. \eqref{S}. If $\ell:\rbb^k\mapsto\rbb_+$ is $G_2$-Lipschitz w.r.t. the $\ell_2$-norm, then
$
\mathfrak{R}_S(\gcal)\leq
\frac{\sqrt{24}RG_2\mathfrak{A}}{n},
$
where
\begin{equation}
\mathfrak{A}=\ebb_{\bm\{\epsilon\}\sim\{\pm1\}^{3nkd}}\ebb\Big[\sup_{f\in\fcal}\sum_{j\in[n]}\sum_{i\in[k]}\sum_{t\in[d]}\Big(\epsilon_{j,i,t,1}f_t(\bx_j)
+\epsilon_{j,i,t,2}f_t(\bx_j^+)+\epsilon_{j,i,t,3}f_t(\bx^-_{ji})\Big)\Big].\label{aa}
\end{equation}
Furthermore, for any $\delta\in(0,1)$, with probability at least $1-\delta$ the following inequality holds for any $f\in\fcal$
\[
L_{un}(f)- \hat{L}_{un}(f)
\leq \frac{4\sqrt{6}RG_2\mathfrak{A}}{n} + 3B\sqrt{\frac{\log(2/\delta)}{2n}}.
\]
\end{theorem}
\begin{remark}
  Under the same Lipschitz continuity w.r.t. $\|\cdot\|_2$, the following bound was established  in \citet{arora2019theoretical}
  \begin{equation}\label{existing-rade}
  L_{un}(f)= \hat{L}_{un}(f)+ O\Big(\frac{G_2R\sqrt{k}\mathfrak{B}}{n}+B\sqrt{\frac{\log(1/\delta)}{n}}\Big),
  \end{equation}
  where
  \begin{equation}
  \mathfrak{B}=\ebb_{\bm{\epsilon}\sim\{\pm1\}^n\times \{\pm1\}^d\times\{\pm1\}^{k+2}}\Big[\sup_{f\in\fcal}\sum_{j\in[n]}\sum_{t\in[d]}\\
  \Big(\epsilon_{j,t,k+1}f_t(\bx_j)+\epsilon_{j,t,k+2}f_t(\bx_j^+)+\sum_{i\in[k]}\epsilon_{j,t,k}f_t(\bx_{ji}^-)\Big)\Big].\label{bb}
  \end{equation}
  Note $\mathfrak{A}$ and $\mathfrak{B}$ are of the same order. Therefore, it is clear that our bound in Theorem \ref{thm:rademacher-l2} improves the bound in \citet{arora2019theoretical} by a factor of $\sqrt{k}$.
\end{remark}

\subsection{$\ell_\infty$ Lipschitz Loss}
We now turn to the analysis for the setting with $\ell_\infty$ Lipschitz continuity assumption, which is more challenging.
The following theorem controls the Rademacher complexity of $\gcal$ w.r.t. the dataset $S$ in terms of the worst-case Rademacher complexity of $\hcal$ defined on the set $S_{\hcal}$,
where
\begin{multline*}
S_{\hcal}=\Big\{\underbrace{(\bx_1,\bx_1^+,\bx^-_{11}),(\bx_1,\bx_1^+,\bx^-_{12}),\ldots,(\bx_1,\bx_1^+,\bx^-_{1k})}_{\text{induced by the first example}},
\underbrace{(\bx_2,\bx_2^+,\bx^-_{21}),(\bx_2,\bx_2^+,\bx^-_{22}),\ldots,(\bx_2,\bx_2^+,\bx^-_{2k})}_{\text{induced by the second example}},
\ldots,\\ \underbrace{(\bx_n,\bx_n^+,\bx^-_{n1}),(\bx_n,\bx_n^+,\bx^-_{n2}),\ldots,(\bx_n,\bx_n^+,\bx^-_{nk})}_{\text{induced by the last example}}\Big\}.
\end{multline*}
As compared to $\gcal$, the function class $\hcal$ removes the loss function $\ell$ and is easier to handle. Our basic idea is to exploit the Lipschitz continuity of $\ell$ w.r.t. $\|\cdot\|_\infty$: to approximate the function class $\{\ell(v_1(\by),\ldots,v_k(\by))\}$, it suffices to approximate each component $v_j(\by),j\in[k]$. This explains why we expand the set $S$ of cardinality $n$ to the set $S_{\hcal}$ of cardinality $nk$.
The proof is given in Section~\ref{sec:proof-gen-infty}.

\begin{theorem}[Complexity bound: $\ell_\infty$-Lipschitz loss\label{thm:rademacher}]
Assume $\|f(\bx)\|_2\leq R$ for any $f\in\fcal$ and $\bx\in\xcal$. Let $S$ be defined as in Eq. \eqref{S}. If $\ell:\rbb^k\mapsto\rbb_+$ is $G$-Lipschitz w.r.t. the $\ell_\infty$-norm, then
\[
\mathfrak{R}_S(\gcal)\leq
24G(R^2+1)n^{-\frac{1}{2}}+ 48G\sqrt{k}\mathfrak{R}_{S_{\hcal},nk}(\hcal)
\times\bigg(1+ \log(4R^2n^{\frac{3}{2}}k)\Big\lceil\log_2\frac{R^2\sqrt{n}}{12}\Big\rceil\bigg),
\]
where
\[
 \mathfrak{R}_{S_{\hcal},nk}(\hcal)=\max\limits_{\big\{(\tilde{\bx}_j,\tilde{\bx}_j^+,\tilde{\bx}_j^-)\big\}_{j\in[nk]}\subseteq S_{\hcal}}
 \ebb_{\bm{\epsilon}\sim\{\pm1\}^{nk}}\Big[\sup_{h\in\hcal}\frac{1}{nk}\sum_{j\in[nk]}\epsilon_jf(\tilde{\bx}_j)^\top(f(\tilde{\bx}_j^+)-f(\tilde{\bx}_j^-))\Big].
\]
\end{theorem}

Note in $\mathfrak{R}_{S_{\hcal},nk}(\hcal)$ we restrict the domain of functions in $\hcal$ to $S_{\hcal}$, and allow an element in $S_{\hcal}$ to be chosen several times in the above maximization.

We can use Lemma \ref{lem:rademacher-H} to control $\mathfrak{R}_{S_{\hcal},nk}(\hcal)$ in Theorem \ref{thm:rademacher}, and derive the following generalization error bound.
The proof is given in Section \ref{sec:proof-gen-infty}.

\begin{theorem}[Generalization bound: $\ell_\infty$-Lipschitz loss\label{thm:gen-unsupervised-loss}]
  Let $\ell:\rbb^k\mapsto\rbb_+$ be $G$-Lipschitz continuous w.r.t. $\|\cdot\|_\infty$. Assume $\|f(\bx)\|_2\leq R,\delta\in(0,1)$. Then with probability at least $1-\delta$ over $S$ for all $f\in\fcal$ we have
  \[
  L_{un}(f)\leq \hat{L}_{un}(f)+ 3B\sqrt{\frac{\log(2/\delta)}{2n}}+48G(R^2+1)n^{-\frac{1}{2}}+
  \frac{96\sqrt{12}GR}{n\sqrt{k}}\bigg(1+ \log(4R^2n^{\frac{3}{2}}k)\Big\lceil\log_2\frac{R^2\sqrt{n}}{12}\Big\rceil\bigg)\mathfrak{C},
  \]
  where
  \begin{equation}
  \mathfrak{C}=\max_{\{(\tilde{\bx}_j,\tilde{\bx}_j^+,\tilde{\bx}_j^-)\}_{j=1}^{nk}\subseteq S_{\hcal}}\ebb_{\bm{\epsilon}\sim\{\pm1\}^{nk}\times\{\pm1\}^d\times\{\pm1\}^3}
  \Big[\sup_{f\in\fcal}\sum_{j\in[nk]}\sum_{t\in[d]}\big(\epsilon_{j,t,1}f_t(\tilde{\bx}_j)+\epsilon_{j,t,2}f_t(\tilde{\bx}_j^+)+\epsilon_{j,t,3}f_t(\tilde{\bx}_j^-)\big)\Big].\label{cc}
  \end{equation}
\end{theorem}

\begin{remark}
  We now compare our bound with Eq. \eqref{existing-rade} developed in \citet{arora2019theoretical}.
  It is reasonable to assume $\mathfrak{C}$ and $\mathfrak{B}$ are of the same order.
  \footnote{Indeed, under a typical behavior of Rademacher complexity as
  $
  \ebb_{\bm{\epsilon}\sim\{\pm\}^n}\sup_{\bm{a}\in\acal\subset\rbb^n}\big[\epsilon_ia_i\big]=O(\sqrt{n})
  $~~\citep{bartlett2002rademacher},
  we have $\mathfrak{C}=O(\sqrt{nkd})$ and $\mathfrak{B}=O(\sqrt{nkd})$.} Then, our bound becomes
  \begin{equation}\label{our-rade}
  L_{un}(f)= \hat{L}_{un}(f)+ O\Big(\frac{GR\mathfrak{B}\log^2(nRk)}{n\sqrt{k}}+B\sqrt{\frac{\log(1/\delta)}{n}}\Big).
  \end{equation}
  We know if $\ell$ is $G_2$-Lipschitz continuous w.r.t. $\|\cdot\|_2$, it is also $\sqrt{k}G_2$-Lipschitz continuous w.r.t. $\|\cdot\|_\infty$. Therefore, in the extreme case we have $G=\sqrt{k}G_2$. Even in this extreme case, our bound is of the order $L_{un}(f)= \hat{L}_{un}(f)+ O\Big(\frac{G_2R\mathfrak{B}\log^2(nRk)}{n}+B\sqrt{\frac{\log(1/\delta)}{n}}\Big)$, which improves Eq. \eqref{existing-rade} by a factor of $\sqrt{k}$ up to a logarithmic factor. For popular loss functions defined in Eq. \eqref{loss-a} and Eq. \eqref{loss-b}, we have $G=G_2=1$ and in this case, our bound in Eq. \eqref{our-rade} improves Eq. \eqref{existing-rade} by a factor of $k$ if we ignore a logarithmic factor.
\end{remark}

\subsection{Self-bounding Lipschitz Loss}
Finally, we consider a self-bounding Lipschitz condition where the Lipschitz constant depends on the loss function values. This definition was given in \citet{reeve2020optimistic}.
\begin{definition}[Self-bounding Lipschitz Continuity]
A loss function $\ell:\rbb^k\mapsto\rbb_+$ is said to be $G_s$-self-bounding Lipschitz continuous w.r.t. $\ell_\infty$  norm if for any $\bm{a},\bm{a}'\in\rbb^k$  
\[
\big|\ell(\bm{a})-\ell(\bm{a}')\big|\leq G_s\max\big\{\ell(\bm{a}),\ell(\bm{a}')\big\}^{\frac{1}{2}}\|\bm{a}-\bm{a}'\|_\infty.
\]
\end{definition}
It was shown that the logistic loss given in Eq. \eqref{loss-b} satisfies the self-bounding Lipschtiz continuity with $G_s=2$~\citep{reeve2020optimistic}.
In the following theorem, we give generalization bounds for learning with self-bounding Lipschitz loss functions. The basic idea is to replace the Lipschitz constant $G$ in Theorem \ref{thm:gen-unsupervised-loss} with empirical errors by using the self-bounding property. We use $\widetilde{O}$ to hide logarithmic factors. The proof is given in Section \ref{sec:proof-gen-self}.
\begin{theorem}[Generalization bound: self-bounding Lipschitz loss\label{thm:gen-self-bounding}]
  Let $\ell:\rbb^k\mapsto\rbb_+$ be $G_s$-self-bounding Lipschitz continuous w.r.t. $\|\cdot\|_\infty$. Assume $\|f(\bx)\|_2\leq R,\delta\in(0,1)$. Then with probability at least $1-\delta$ over $S$ we have the following inequality uniformly for all $f\in\fcal$
  \begin{multline*}
  L_{un}(f)= \hat{L}_{un}(f)+\widetilde{O}\Big((B+G_s^2R^4)n^{-1}+G_s^2R^2n^{-2}k^{-1}\mathfrak{C}^2\Big)
\\+ \widetilde{O}\Big((\sqrt{B}+G_sR^2)n^{-\frac{1}{2}}+G_sRn^{-1}k^{-\frac{1}{2}}\mathfrak{C}\Big)\hat{L}^{\frac{1}{2}}_{un}(f).
  \end{multline*}
\end{theorem}
\begin{remark}
  Theorem \ref{thm:gen-self-bounding} gives optimistic generalization bounds in the sense that the upper bounds depend on empirical errors~\citep{srebro2010smoothness}. Therefore, the generalization bounds for $L_{un}(f)- \hat{L}_{un}(f)$ would benefit from low training errors. In particular, if $\hat{L}^{\frac{1}{2}}_{un}(f)=0$, Theorem \ref{thm:gen-self-bounding} implies generalization bounds
  \[
  L_{un}(f)= \hat{L}_{un}(f)\!+\!\widetilde{O}\Big(Bn^{-1}\!+\!G_s^2R^4n^{-1}\!+\!G_s^2R^2n^{-2}k^{-1}\mathfrak{C}^2\Big).
  \]
  Typically, we have $\mathfrak{C}=O(\sqrt{nk})$ and in this case $L_{un}(f)= \hat{L}_{un}(f)+\widetilde{O}\big(Bn^{-1}+G_s^2R^4n^{-1}\big)$. In other words, we get fast-decaying error bounds in an interpolating setting.
\end{remark}

\section{Applications\label{sec:app}}
To apply Theorem \ref{thm:rademacher-l2} and Theorem \ref{thm:gen-unsupervised-loss}, we need to control the term $\mathfrak{A}$ or $\mathfrak{C}$, which is related to the Rademacher complexity of a function class. In this section, we will show how to control $\mathfrak{C}$ for features of the form $\bx\mapsto U\bv(\bx)$, where $U\in\rbb^{d\times d'}$ is a matrix and $\bv:\xcal\mapsto\rbb^{d'}$. Here $\bv$ maps the original data $\bx\in\xcal$ to an intermediate feature in $\rbb^{d'}$, which is used for all the final features. If $\bv$ is the identity map, then we get linear features. If $\bv$ is a neural network, then we get nonlinear features.
For a norm $\|\cdot\|$ on a matrix, we denote by $\|\cdot\|_*$ its dual norm.
The following lemmas to be proved in Section \ref{sec:proof-rade} give general results on Rademacher complexities. Lemma \ref{lem:rademacher-feature-space} gives upper bounds, while Lemma \ref{lem:lower} gives lower bounds.  It is immediate to extend our analysis to control $\mathfrak{A}$. For brevity we ignore such a discussion.
\begin{lemma}[Upper bound\label{lem:rademacher-feature-space}]
Let $d,d'\in\nbb$.
Let $\vcal$ be a class of functions from $\xcal$ to $\rbb^{d'}$.
Let $\fcal=\{f(\bx)=U\bv(\bx):U\in\ucal,\bv\in\vcal\}$, where $\ucal=\big\{U=(\bu_1,\ldots,\bu_d)^\top\in\rbb^{d\times d'}:\|U^\top\|\leq\Lambda\big\}$
and
$
f(\bx)=U\bv(\bx)=(\bu_1,\ldots,\bu_d)^\top\bv(\bx).
$
Then
\[\ebb_{\bm{\epsilon}\sim\{\pm1\}^{nd}}\sup_{f\in\fcal}\sum_{t\in[d]}\sum_{j\in[n]}\epsilon_{j,t}f_t(\bx_j)\leq
\Lambda\ebb_{\bm{\epsilon}\sim\{\pm1\}^{nd}}\sup_{\bv\in\vcal}\big\|\big(\sum_{j\in[n]}\epsilon_{1,j}\bv(\bx_j),\ldots,\sum_{j\in[n]}\epsilon_{d,j}\bv(\bx_j)\big)\big\|_*.
\]
\end{lemma}

\begin{lemma}[Lower bound\label{lem:lower}]
If $\fcal$ is symmetric in the sense that $f\in\fcal$ implies $-f\in\fcal$, then we have
\[
\mathfrak{C}\geq \sqrt{\frac{nk}{2}}\!\sup\limits_{f\in\fcal}\max\limits_{(\tilde{\bx},\tilde{\bx}^+,\tilde{\bx}^-)\in  S_{\hcal}}\!\!\!
  \Big(\|f(\tilde{\bx})\|_2^2+\|f(\tilde{\bx}^+)\|_2^2+\|f(\tilde{\bx}^-)\|_2^2\Big)^{\frac{1}{2}}.
\]
\end{lemma}

\subsection{Linear Features}
We first apply Lemma \ref{lem:rademacher-feature-space} to derive Rademacher complexity bounds for learning with linear features.
For any $p\geq1$ and a matrix $W=(\bw_1,\ldots,\bw_{d'})\in\rbb^{d\times d'}$, the $\ell_{2,p}$ norm of $W$ is defined as
$\|W\|_{2,p}=\big(\sum_{i\in[d']}\|\bw_i\|_2^p\big)^{\frac{1}{p}}$. If $p=2$, this becomes the Frobenius norm $\|W\|_F$.
For any $p\geq1$, the Schatten-$p$ norm of a matrix $W\in\mathbb R^{d\times d'}$ is defined as the $\ell_p$-norm of the
vector of singular values $(\sigma_1(W),\ldots,\sigma_{\min\{d,d'\}}(W))^\top$ (the singular values are assumed to be sorted in non-increasing order), i.e.,
$\|W\|_{S_p}:=\|\sigma(W)\|_p$. Let $p^*$ be the number satisfying $1/p+1/p^*=1$. The following proposition to be proved in Section \ref{sec:proof-rade-linear} gives complexity bounds for learning with linear features.
\begin{proposition}[Linear representation\label{prop:rade-linear}]
  Consider the feature map defined in Lemma \ref{lem:rademacher-feature-space} with $\bv(\bx)=\bx$.
  \begin{enumerate}[(a)]
    \item  If $\|\cdot\|=\|\cdot\|_{2,p}$, then
   \[
    \ebb_{\bm{\epsilon}} \sup\limits_{f\in\fcal}\sum\limits_{t\in[d]}\sum\limits_{j\in[n]}\epsilon_{t,j}f_t(\bx_j)\leq \min_{q\geq p}\Big\{\Lambda d^{1/q^*}\max(\sqrt{q^*-1},1)\Big\}\Big(\sum_{j\in[n]}\|\bx_j\|_2^2\Big)^{\frac{1}{2}}.
   \]
    \item If $\|\cdot\|=\|\cdot\|_{S_p}$ with $p\leq 2$, then
    \begin{multline*}
    \ebb_{\bm{\epsilon}} \sup_{f\in\fcal}\sum_{t\in[d]}\sum_{j\in[n]}\epsilon_{t,j}f_t(\bx_j)
    \leq
    \Lambda2^{-\frac{1}{4}}\min_{q\in[p,2]}\sqrt{\frac{q^*\pi}{e}}
    \max\Big\{\Big\|\Big(d\sum_{j\in[n]}\bx_j\bx_j^\top\Big)^{\frac{1}{2}}\Big\|_{S_{q^*}},d^{1/q^*}\Big(\sum_{j\in[n]}\|\bx_j\|_2^2\Big)^{1/2}\Big\}.
    \end{multline*}
  \end{enumerate}
\end{proposition}

We now plug the above Rademacher complexity bounds into Theorem \ref{thm:gen-unsupervised-loss} to give generalization error bounds for learning with unsupervised loss.
Let $B_x=\max\{\|\bx_j\|_2,\|\bx_j^+\|_2,\|\bx^-_{jt}\|_2:j\in[n],t\in[k]\}$.
Note
$\big(\sum_{j\in[nk]}\|\tilde{\bx}_j\|_2^2\big)^{\frac{1}{2}}\leq \sqrt{nk}B_x$ for $\tilde{\bx}_j$ in the definition of $\mathfrak{C}$, from which and Proposition \ref{prop:rade-linear} we get the following bound for the case $\bv(\bx)=\bx$ (the definition of $\mathfrak{C}$ involves $nk$ examples, while in Proposition \ref{prop:rade-linear} we consider $n$ examples):
\[\mathfrak{C}=O\Big(\sqrt{nk}B_x\min_{q\geq p}\big\{\Lambda d^{1/q^*}\max(\sqrt{q^*-1},1)\big\}\Big).\]
The following corollary then follows from Theorem~\ref{thm:gen-unsupervised-loss}.
\begin{corollary}\label{cor:gen-unsupervised-linear}
  Consider the feature map in Proposition \ref{prop:rade-linear} with $\|\cdot\|=\|\cdot\|_{2,p}$. Let $\ell$ be the logistic loss and $\delta\in(0,1)$.
    Then with probability at least $1-\delta$ the following inequality holds for all $f\in\fcal$
    \[
    L_{un}(f)- \hat{L}_{un}(f)=\frac{B\log^{\frac{1}{2}}(1/\delta)}{\sqrt{n}}+
    \widetilde{O}\Big(\frac{GRB_x\min_{q\geq p}\big\{\Lambda d^{1/q^*}\max(\sqrt{q^*-1},1)\big\}}{\sqrt{n}}\Big).
    \]
\end{corollary}

It is also possible to give generalization bounds for learning with $\ell_2$-Lipschitz loss functions, and optimistic generalization bounds for learning with self-bounding Lipschitz loss functions. We omit the discussion for brevity.

\subsection{Nonlinear Features}
We now consider Rademacher complexity for learning with nonlinear features by DNNs.
The following lemma to be proved in Section \ref{sec:proof-rade-dnn} gives Rademacher complexity bounds for learning with features by DNNs.
We say an activation $\sigma:\rbb\mapsto\rbb$ is positive-homogeneous if $\sigma(ax)=a\sigma(x)$ for $a\geq0$, contractive if $|\sigma(x)-\sigma(x')|\leq|x-x'|$. The ReLU activation function $\sigma(x)=\max\{x,0\}$ is both positive-homogeneous and contractive.
\begin{proposition}[Nonlinear representation\label{prop:rade-dnn}]
  Consider the feature map defined in Lemma \ref{lem:rademacher-feature-space} with $\|\cdot\|=\|\cdot\|_F$ and
\[
\vcal=\Big\{\bx\mapsto \bv(x)=\sigma\big(V_{L}\sigma\big(V_{L-1}\cdots\sigma(V_1\bx)\big)\big):
\|V_l\|_F\leq B_l,\forall l\in[L]\Big\},
\]
where $\sigma$ is positive-homogeneous, contractive and $\sigma(0)=0$, and $L$ is the number of layers.
Then
\[
\ebb_{\bm{\epsilon}} \sup_{f\in\fcal}\sum_{t\in[d]}\sum_{j\in[n]}\epsilon_{t,j}f_t(\bx_j)\leq
      \sqrt{d}\Lambda B_LB_{L-1}\cdots B_1
      \bigg(16L\big(\sum_{1\leq i<j\leq n}(\bx_i^\top\bx_j)^2\big)^{\frac{1}{2}}+\sum_{j\in[n]}\|\bx_j\|_2^2\bigg)^{\frac{1}{2}}.
\]
\end{proposition}

\begin{remark}
  If $d=1$, the following bound was established in \citet{golowich2018size}
  \[
  \ebb_{\bm{\epsilon}} \sup_{f\in\fcal}\sum_{j\in[n]}\epsilon_{j}f_t(\bx_j)=O\Big(\sqrt{L}\Big(\sum_{j\in[n]}\|\bx_j\|_2^2\Big)^{\frac{1}{2}}\prod_{l\in[L]}B_l\Big).
  \]
  Proposition \ref{prop:rade-dnn} extends this bound to the general case $d\in\nbb$. In particular, if $d=1$, our result matches the result in \citet{golowich2018size} within a constant factor.
  We need to introduce different techniques to handle the difficulty in considering the coupling among different features $\bm{u}_t^\top\bv(x),t\in[d]$, which is reflected by the regularizer on $U$ as $\|U\|_F\leq \Lambda$. Ignoring this coupling would imply a bound with a crude dependency on $d$. For example, we consider the moment generation function of $\sup_{f\in\fcal}\big(\sum_{t\in[d]}\sum_{j\in[n]}\epsilon_{t,j}f_t(\bx_j)\big)^2$, and then reduce it to the moment generation function of a Rademacher chaos variable $\sum_{1\leq i<j\leq n}\epsilon_i\epsilon_j\bx_i^\top\bx_j$ by repeated applications of contraction inequalities of Rademacher complexities. A direct application of the analysis in \citet{golowich2018size} would imply bounds $\ebb_{\bm{\epsilon}} \sup_{f\in\fcal}\sum_{j\in[n]}\epsilon_{j}f_t(\bx_j)=O\big(d\sqrt{L}\big(\sum_{j\in[n]}\|\bx_j\|_2^2\big)^{\frac{1}{2}}\prod_{l\in[L]}B_l\big)$.
  As a comparison, our analysis implies a bound with a square-root dependency on $d$. We will give more details on the comparison of technical analysis in Remark \ref{rem:golowich}.
\end{remark}

Note $\big(\sum_{1\leq i<j\leq n}(\bx_i^\top\bx_j)^2\big)^{\frac{1}{2}}=O\big(\sum_{j\in[n]}\|\bx_j\|_2^2\big)$, from which and Proposition \ref{prop:rade-dnn} we get for nonlinear features that
$\mathfrak{C}=O\big(\sqrt{dL}\Lambda(nk)^{\frac{1}{2}}B_x\prod_{l\in[L]}B_l\big).$ The following proposition then follows directly from Theorem \ref{thm:gen-unsupervised-loss}.
\begin{corollary}\label{cor:gen-unsupervised-dnn}
  Consider the feature map in Proposition \ref{prop:rade-dnn}. Let $\ell$ be the logistic loss and $\delta\in(0,1)$. With probability at least $1-\delta$ the following inequality holds for all $f\in\fcal$
  \[
  L_{un}(f)-\hat{L}_{un}(f)=
  \widetilde{O}\Big(\frac{GR\sqrt{dL}\Lambda B_x\prod_{l\in[L]}B_l+B\log^{\frac{1}{2}}\frac{1}{\delta}}{\sqrt{n}}\Big).
  \]
\end{corollary}

\subsection{Generalization for Downstream Classification}

In this subsection, we apply the above generalization bounds on unsupervised learning to derive generalization guarantees for a downstream supervised learning task.
Similar ideas can be dated back to metric/similarity learning, where one shows that similarity-based learning guarantees a good generalization of
the resultant classification~\citep{guo2014guaranteed,balcan2008,balcan2006theory}.
Following \citet{arora2019theoretical}, we consider a particular \emph{mean classifier} with rows being the means of the representation of each class, i.e., $\bx\mapsto W^\mu f(\bx)$ with the $c$-th row of $W$ being the mean $\mu_c$ of representations of inputs with label $c$: $\mu_c:=\ebb_{\bx\sim\dcal_c}[f(\bx)]$. 
Consider the \emph{average supervised loss}
\[
L^\mu_{sup}(f):=\ebb_{\{c_i\}_{i=1}^{\kk+1}\sim\rho^{\kk+1}}\big[L_{sup}(\{c_i\}_{i=1}^{\kk+1},W^\mu f)|c_i\neq c_j\big],
\]
where we take the expectation over $\tcal=\{c_i\}_{i=1}^{\kk+1}$.
The following lemma shows that the generalization performance of the mean classifier based on a representation $f$ can be guaranteed in terms of the generalization performance of the representation in unsupervised learning.
\begin{lemma}[\citealt{arora2019theoretical}\label{lem:gen-class}]
  There exists a function $\rho:\ccal^{\kk+1}\mapsto\rbb_+$ such that the following inequality holds for any $f\in\fcal$:
  $
  \ebb_{\tcal\sim\dcal}\big[\rho(\tcal)L^\mu_{sup}(f)\big]\leq L_{un}(f).
  $
\end{lemma}
We refer the interested readers to \citet{arora2019theoretical} for the expression of $\rho(\tcal)$, which is independent of $n$. The following corollaries are immediate applications of Lemma~\ref{lem:gen-class} and our generalization bounds for unsupervised learning. We omit the proof for brevity.

\begin{corollary}[Linear representation\label{cor:class-linear}]
  Consider the feature map in Proposition \ref{prop:rade-linear} with $\|\cdot\|=\|\cdot\|_{2,p}$. Let $\ell_s, \ell$ be the logistic loss and $\delta\in(0,1)$.
  Then with probability at least $1-\delta$ the following inequality holds
    \[
    \ebb_{\tcal\sim\dcal}\Big[\rho(\tcal)L^\mu_{sup}(\hat{f})\Big]=\hat{L}_{un}(\hat{f})+\widetilde{O}\Big(\frac{B\log^{\frac{1}{2}}(1/\delta)}{\sqrt{n}}
    +\frac{GRB_x\min_{q\geq p}\big\{\Lambda d^{1/q^*}\max(\sqrt{q^*-1},1)\big\}}{\sqrt{n}}\Big).
    \]
\end{corollary}

\begin{remark}
If $p\leq (\log d)/(\log d -1)$, we set $q=(\log d)/(\log d-1)$, and get $d^{1/q^*}\max(\sqrt{q^*-1},1)=O(\log^{\frac{1}{2}} d)$ ($q^*=\log d$). In this case, we get a bound with a logarithmic dependency on the number of features. It is possible to extend our discussion to more general norms $\|\cdot\|=\|\cdot\|_{p,q},p,q\geq1$~\citep{kakade2012regularization}.
\end{remark}

\begin{corollary}[Nonlinear representation\label{cor:class-dnn}]
  Consider the feature map in Proposition \ref{prop:rade-dnn}. Let $\ell_s, \ell$ be the logistic loss and $\delta\in(0,1)$. With probability at least $1-\delta$ we have
  \[
  \ebb_{\tcal\sim\dcal}\Big[\rho(\tcal)L^\mu_{sup}(\hat{f})\Big]=\hat{L}_{un}(\hat{f})+
  \widetilde{O}\Big(\frac{GR\sqrt{dL}\Lambda B_x\prod_{l\in[L]}B_l}{\sqrt{n}}+\frac{B\log^{1/2}(1/\delta)}{\sqrt{n}}\Big).
  \]
\end{corollary}

\begin{remark}
If we combine our Rademacher complexity bounds in Section \ref{sec:app} and Eq. \eqref{existing-rade} developed in \citet{arora2019theoretical}, we would get generalization bounds for supervised classification with a linear dependency on $k$. If we combine our complexity bounds and Theorem \ref{thm:rademacher-l2}, we would get generalization bounds for supervised classification with a square-root dependency on $k$. These discussions use the Lipschitz continuity of $\ell$ w.r.t $\|\cdot\|_2$. As a comparison, the use of Lipschitz continuity w.r.t. $\|\cdot\|_\infty$ allows us to derive generalization bounds with a logarithmic dependency on $k$ in Corollary \ref{cor:class-linear} and Corollary \ref{cor:class-dnn}. Furthermore, we can improve the bounds $\widetilde{O}(1/\sqrt{n})$ in these corollaries to $\widetilde{O}(1/n)$ in an interpolation setting by applying Theorem~\ref{thm:gen-self-bounding}.
\end{remark}

\section{Conclusion\label{sec:conclusion}}
Motivated by the existing generalization bounds with a crude dependency on the number of negative examples, we present a systematic analysis on the generalization behavior of CRL. We consider three types of loss functions. Our results improve the existing bounds by a factor of $\sqrt{k}$ for $\ell_2$ Lipschitz loss, and by a factor of $k$ for $\ell_\infty$ Lipschitz loss (up to a log factor). We get optimistic bounds for self-bounding Lipschitz loss, which imply fast rates under low noise conditions. We justify the effectiveness of our results with applications to both linear and nonlinear features.

Our analysis based on Rademacher complexities implies algorithm-independent bounds. It would be interesting to develop algorithm-dependent bounds to understand the interaction between optimization and generalization. For $\ell_\infty$ loss, our bound still enjoys a logarithmic dependency on $k$. It would be interesting to study whether this logarithmic dependency can be removed in further study.

\setlength{\bibsep}{0.16cm}

\appendix
\numberwithin{equation}{section}
\numberwithin{theorem}{section}
\numberwithin{figure}{section}
\numberwithin{table}{section}
\renewcommand{\thesection}{{\Alph{section}}}
\renewcommand{\thesubsection}{\Alph{section}.\arabic{subsection}}
\renewcommand{\thesubsubsection}{\Roman{section}.\arabic{subsection}.\arabic{subsubsection}}
\setcounter{secnumdepth}{-1}
\setcounter{secnumdepth}{3}

\section{Proof of Theorem \ref{thm:rademacher-l2}\label{sec:proof-rademacher-l2}}
To prove Theorem \ref{thm:rademacher-l2}, we first prove Lemma \ref{lem:rademacher-H} by the following vector-contraction lemma on Rademacher complexities. 
\begin{lemma}[\citealt{maurer2016vector}\label{lem:maurer}]
  Let $S=\{\bz_j\}_{j=1}^n\in\zcal^n$. Let $\fcal'$ be a class of functions $f':\zcal\mapsto\rbb^d$ and $h:\rbb^d\mapsto\rbb$ be $G$-Lipschitz w.r.t. $\ell_2$-norm.  Then
  \[
  \ebb_{\bm{\epsilon}\sim\{\pm1\}^n}\Big[\sup_{f'\in\fcal'}\sum_{j\in[n]}\epsilon_j(h\circ f')(\bz_j)\Big]
  \leq\sqrt{2}G\ebb_{\bm{\epsilon}\sim\{\pm1\}^{nd}}\Big[\sup_{f'\in\fcal'}\sum_{j\in[n]}\sum_{t\in[d]}\epsilon_{j,t}f'_t(\bz_j)\Big].
  \]
\end{lemma}
In Section \ref{sec:gen-cont}, we will provide an extension of the above lemma.
\begin{proof}[Proof of Lemma \ref{lem:rademacher-H}]
Let $f':\xcal^3\mapsto\rbb^{3d}$ be defined as
\[
f'(\bx,\bx^+,\bx^-)=\big(f(\bx),f(\bx^+),f(\bx^-)\big)\in\rbb^{3d}
\]
and $h:\rbb^{3d}\mapsto\rbb$ be defined as
\[
h(\by,\by^+,\by^-)=\by^\top(\by^+-\by^-),\quad\by,\by^+,\by^-\in\rbb^d.
\]
Then it is clear that
\[
f(\bx)^\top(f(\bx^+)-f(\bx^-))=h\circ f'(\bx,\bx^+,\bx^-).
\]
Furthermore, for any $\by_1,\by^+_1,\by^-_1,\by_2,\by^+_2,\by^-_2$ with Euclidean norm less than or equal to $R$, we have
\begin{align*}
   & h(\by_1,\by_1^+,\by_1^-)-h(\by_2,\by_2^+,\by_2^-) = \by_1^\top(\by_1^+-\by_1^-) - \by_2^\top(\by_2^+-\by_2^-) \\
   & = \by_1^\top(\by_1^+-\by_1^-) - \by_1^\top(\by_2^+-\by_2^-) + \by_1^\top(\by_2^+-\by_2^-) - \by_2^\top(\by_2^+-\by_2^-) \\
   & = \by_1^\top(\by_1^+-\by_1^--\by_2^++\by_2^-)+(\by_1-\by_2)^\top(\by_2^+-\by_2^-).
\end{align*}
It then follows from the elementary inequality $(a+b)^2\leq (1+p)a^2+(1+1/p)b^2$ that
\begin{multline*}
  \big|h(\by_1,\by_1^+,\by_1^-)-h(\by_2,\by_2^+,\by_2^-)\big|^2 \leq \\
  2(1+p)\|\by_1\|^2\|\by_1^+-\by_2^+\|^2+2(1+p)\|\by_1\|^2\|\by_1^--\by_2^-\|_2^2+(1+1/p)\|\by_2^+-\by_2^-\|_2^2\|\by_1-\by_2\|_2^2.
\end{multline*}
We can choose $p=2$ and get
\begin{align*}
\big|h(\by_1,\by_1^+,\by_1^-)-h(\by_2,\by_2^+,\by_2^-)\big|^2 &\leq 6R^2\big(\|\by_1^+-\by_2^+\|^2+\|\by_1^--\by_2^-\|_2^2+\|\by_1-\by_2\|_2^2\big)\\
& = 6R^2\|(\by_1,\by_1^+,\by_1^-)-(\by_2,\by_2^+,\by_2^-)\|_2^2.
\end{align*}
This shows that $h$ is $\sqrt{6}R$-Lipschitz continuous w.r.t. $\|\cdot\|_2$. We can apply Lemma \ref{lem:maurer} to derive
\begin{multline*}
\ebb_{\bm{\epsilon}\sim\{\pm1\}^n}\Big[\sup_{f'\in\fcal'}\sum_{j\in[n]}\epsilon_j(h\circ f')(\bz_j)\Big]\\
\leq\sqrt{12}R\ebb_{\bm{\epsilon}\sim\{\pm1\}^n\times\{\pm1\}^d\times\{\pm1\}^3}\Big[\sup_{f\in\fcal}\sum_{j\in[n]}\sum_{t\in[d]}\Big(\epsilon_{j,t,1}f_t(\bx_j)+
  \epsilon_{j,t,2}f_t(\bx_j^+)+\epsilon_{j,t,3}f_t(\bx_j^-)\Big)\Big].
\end{multline*}
The proof is completed.
\end{proof}

The following standard lemma gives generalization error bounds in terms of Rademacher complexities.

\begin{lemma}[\label{lem:gen-rademacher}\citealt{mohri2012foundations}]
Let $\widetilde{\gcal}$ be a function class and $\widetilde{S}=\{\bz_1,\ldots,\bz_n\}$. If for any $g\in \widetilde{\gcal}$ we have $g(\bz)\in[0,B]$, then for any $\delta\in(0,1)$ the following inequality holds with probability (w.r.t. $\widetilde{S}$) at least $1-\delta$
\[
\ebb[g(\bz)]\leq \frac{1}{n}\sum_{i=1}^{n}g(\bz_i)+2\mathfrak{R}_{\widetilde{S}}(\widetilde{\gcal})+3B\sqrt{\frac{\log(2/\delta)}{2n}},\quad\forall g\in \widetilde{\gcal}.
\]
\end{lemma}

\begin{proof}[Proof of Theorem \ref{thm:rademacher-l2}]
  According to the $G_2$-Lipschitz continuity of $\ell$ w.r.t. $\ell_2$-norm and Lemma \ref{lem:maurer}, we have
  \begin{multline*}
  \ebb_{\bm{\epsilon}\sim\{\pm1\}^n}\Big[\sup_{f\in\fcal}\sum_{j\in[n]}\epsilon_j\ell\big(\big\{f(\bx_j)^\top(f(\bx_j^+)-f(\bx^-_{ji}))\big\}_{i\in[k]}\big)\Big]\\
  \leq \sqrt{2}G_2\ebb_{\bm\{\epsilon\}\sim\{\pm1\}^{nk}}\ebb\Big[\sup_{f\in\fcal}\sum_{j\in[n]}\sum_{i\in[k]}\epsilon_{j,i}f(\bx_j)^\top(f(\bx_j^\top)-f(\bx^-_{ji}))\Big].
  \end{multline*}
  According to Lemma \ref{lem:rademacher-H}, we further get
  \begin{multline*}
    \ebb_{\bm\{\epsilon\}\sim\{\pm1\}^{nk}}\ebb\Big[\sup_{f\in\fcal}\sum_{j\in[n]}\sum_{i\in[k]}\epsilon_{j,i}f(\bx_j)^\top(f(\bx_j^\top)-f(\bx^-_{ji}))\Big] \leq \\
    \sqrt{12}R\ebb_{\bm\{\epsilon\}\sim\{\pm1\}^{3nkd}}\ebb\Big[\sup_{f\in\fcal}\sum_{j\in[n]}\sum_{i\in[k]}\sum_{t\in[d]}\Big(\epsilon_{j,i,t,1}f_t(\bx_j)+\epsilon_{j,i,t,2}f_t(\bx_j^+)+\epsilon_{j,i,t,3}f_t(\bx^-_{ji})\Big)\Big].
  \end{multline*}
  We can combine the above two inequalities to get the Rademacher complexity bounds.

  We now turn to the generalization bounds. Applying Lemma \ref{lem:gen-rademacher}, with probability at least $1-\delta$ the following inequality holds with probability at least $1-\delta$
  \[
  L_{un}(f)\leq \hat{L}_{un}(f)+2\mathfrak{R}_S(\gcal)+3B\sqrt{\frac{\log(2/\delta)}{2n}},\quad\forall f\in\fcal.
  \]
  The stated bound on generalization errors then follows by plugging the Rademacher complexity bounds into the above bound. The proof is completed.
\end{proof}

\section{Proof of Theorem \ref{thm:gen-unsupervised-loss}\label{sec:proof-gen-infty}}

We first introduce several complexity measures such as covering numbers and fat-shattering dimension~\citep{alon1997scale,anthony1999neural,zhou2002covering}.
\begin{definition}[Covering number]\label{def:covering-number}
  Let $\widetilde{S}:=\{\bz_1,\ldots,\bz_n\}\in\zcal^n$. Let $\widetilde{\fcal}$ be a class of real-valued functions defined over a space ${\zcal}$.
  For any $\epsilon>0$ and $p\geq1$, the empirical $\ell_p$-norm covering number
  $\ncal_p(\epsilon,\widetilde{\fcal},\widetilde{S})$ with respect to $\widetilde{S}$ is defined as the smallest number $m$ of a collection of vectors $\bv^1,\ldots,\bv^m\in\{(f(\bz_1),\ldots,f(\bz_n)):f\in\widetilde{\fcal}\}$
  such that
  $$
    \sup_{f\in \widetilde{\fcal}}\min_{j\in[m]}\Big(\frac{1}{n}\sum_{i\in[n]}|f(\bz_i)-\bv_i^j|^p\Big)^{\frac{1}{p}}\leq\epsilon,
  $$
  where $\bv_i^j$ is the $i$-th component of the vector $\bv^j$.
  In this case, we call $\{\bv^1,\ldots,\bv^m\}$ an $(\epsilon,\ell_p)$-cover of $\widetilde{\fcal}$ with respect to $\widetilde{S}$.
\end{definition}
\begin{definition}[Fat-Shattering Dimension]\label{def:shattering}
  Let $\widetilde{\fcal}$ be a class of real-valued functions defined over a space $\widetilde{\zcal}$.
  We define the fat-shattering dimension $\fat_\epsilon(\widetilde{\fcal})$ at scale $\epsilon>0$ as the largest $D\in\nbb$ such that there exist $D$ points $\bz_1,\ldots,\bz_D\in\tilde{\zcal}$ and witnesses $s_1,\ldots,s_D\in\rbb$ satisfying: for any $\delta_1,\ldots,\delta_D\sim\{\pm1\}$ there exists $f\in \widetilde{\fcal}$ with
  $$
    \delta_i(f(\bz_i)-s_i)\geq\epsilon/2,\qquad\forall i\in[D].
  $$
\end{definition}
To prove Theorem \ref{thm:rademacher}, we need to introduce the following lemma on Rademacher complexity, fat-shattering dimension and covering numbers. Part (a) shows that the covering number can be bounded by fat-shattering dimension (see, e.g., Theorem 12.8 in \citet{anthony1999neural}). Part (b) shows that the fat-shattering dimension can be controlled by the worst-case Rademacher complexity, which was developed in \citet{srebro2010smoothness}. Part (c) is a discretization of the chain integral to control Rademacher complexity by covering numbers~\citep{srebro2010smoothness}, which can be found in \citet{guermeur2017lp}.
Let $e$ be the base of the natural logarithms.
\begin{lemma}\label{lem:shattering}
  Let $\widetilde{S}:=\{\bz_1,\ldots,\bz_n\}\subseteq\tilde{\zcal}$. Let $\widetilde{\fcal}$ be a class of real-valued functions defined over a space $\widetilde{\zcal}$.
\begin{enumerate}[(a)]
  \item If functions in $\widetilde{\fcal}$ take values in  $[-B,B]$, then for any $\epsilon>0$ with $\fat_\epsilon(\widetilde{\fcal})<n$ we have
  $$
    \log\ncal_\infty(\epsilon,\widetilde{\fcal},\widetilde{S})\leq 1+ \fat_{\epsilon/4}(\widetilde{\fcal})\Big(\log_2\frac{4eBn}{\epsilon\fat_{\epsilon/4}(\widetilde{\fcal})}\Big)\Big(\log\frac{16B^2n}{\epsilon^2}\Big).
  $$
  \item For any $\epsilon> \frak{R}_{\widetilde{\zcal},n}(\widetilde{\fcal})$, we have
  $\fat_\epsilon(\widetilde{\fcal})<\frac{4n}{\epsilon^2}\frak{R}_{\widetilde{\zcal},n}^2(\widetilde{\fcal})$.
  \item Let $(\epsilon_j)_{j=0}^\infty$ be a monotone sequence decreasing to $0$ and any $(a_1,\ldots,a_n)\in\rbb^n$. If \[\epsilon_0\geq \sqrt{n^{-1}\sup_{f\in\widetilde{\fcal}}\sum_{i=1}^n\big(f(\bz_i)-a_i\big)^2},\] then for any non-negative integer $N$ we have
      \begin{equation}\label{entropy-integral}
        \mathfrak{R}_{\widetilde{S}}(\widetilde{\fcal})\leq 2\sum_{j=1}^N\big(\epsilon_j+\epsilon_{j-1})\sqrt{\frac{\log\ncal_\infty(\epsilon_j,\widetilde{\fcal},\widetilde{S})}{n}}+\epsilon_N.
      \end{equation}
\end{enumerate}
\end{lemma}
According to Part (a) of Lemma \ref{lem:shattering}, the following inequality holds for any $\epsilon\in(0,2B]$ (the case $\fat_{\epsilon/4}(\widetilde{\fcal})=0$ is trivial since in this case we have $\ncal_\infty(\epsilon,\widetilde{\fcal},\widetilde{S})=1$, and otherwise we have $\fat_{\epsilon/4}(\widetilde{\fcal})\geq 1$)
\begin{equation}\label{shattering}
 \log\ncal_\infty(\epsilon,\widetilde{\fcal},\widetilde{S})\leq 1+ \fat_{\epsilon/4}(\widetilde{\fcal})\log_2^2\frac{8eB^2|\widetilde{S}|}{\epsilon^2}.
\end{equation}

We follow the arguments in \citet{lei2019data} to prove Theorem \ref{thm:rademacher}.
\begin{proof}[Proof of Theorem \ref{thm:rademacher}]
We first relate the empirical $\ell_\infty$-covering number of $\fcal$ w.r.t. $S=\{(\bx_j,\bx_j^+,\bx^-_{j1},\bx^-_{j2},\ldots,\bx^-_{jk}):j\in[n]\}$ to the empirical $\ell_\infty$-covering number of
$\hcal$ w.r.t. $S_{\hcal}$.
Let
\[
 \Big\{\mathbf{r}^m = \big(r^m_{1,1},r^m_{1,2},\ldots,r^m_{1,k},\ldots,r^m_{n,1},r^m_{n,2},\ldots,r^m_{n,k}\big)\\
      :m\in[N]\Big\}
\]
be an $(\epsilon/G,\ell_\infty)$-cover of $\hcal$ w.r.t. $S_{\hcal}$.
Recall that
\begin{equation}\label{hf}
  h_f(\bx,\bx^+,\bx^-)=f(\bx)^\top\big(f(\bx^+)-f(\bx^-)\big).
\end{equation}
Then, by the definition of $\ell_\infty$-cover we know for any $f\in\fcal$ we can find $m\in[N]$ such that
\[
\max_{j\in[n]}\max_{i\in[k]}\big|h_f(\bx_j,\bx_j^+,\bx_{ji}^-)-r^m_{j,i}\big|\leq\epsilon/G.
\]
By the Lipschitz continuity of $\ell$, we then get
\begin{align*}
& \max_{j\in[n]}\big|\ell\big(\big\{f(\bx_j)^\top\big(f(\bx_j^+)-f(\bx_{ji}^-)\big)\big\}_{i=1}^k\big)-\ell\big(\{r^m_{j,i}\}_{i=1}^k\big)\big|\\
& \leq G\big\|\big(f(\bx_j)^\top\big(f(\bx_j^+)-f(\bx_{ji}^-)\big)\big)_{i=1}^k-\big(r^m_{j,i}\big)_{i=1}^k\big\|_\infty = G\big\|\big(h_f(\bx_j,\bx_j^+,\bx^-_{ji})\big)_{i=1}^k-\big(r^m_{j,i}\big)_{i=1}^k\big\|_\infty\\
& \leq G\epsilon/G=\epsilon.
\end{align*}
This shows that $\big\{\big(\ell\big(\{r^m_{1,i}\}_{i=1}^k\big),\ell\big(\{r^m_{2,i}\}_{i=1}^k\big),\ldots,\ell\big(\{r^m_{n,i}\}_{i=1}^k\big)\big):m\in[N]\big\}$ is an $(\epsilon,\ell_\infty)$-cover of $\gcal$ w.r.t. $S$ and therefore
\begin{equation}\label{cover-1}
\ncal_\infty(\epsilon,\gcal,S)\leq\ncal_\infty(\epsilon/G,\hcal,S_{\hcal}).
\end{equation}
Since we consider empirical covering number of $\fcal$ w.r.t. $S$, we can assume functions in $\hcal$ are defined over $S_{\hcal}$.
For simplicity, we denote $\mathfrak{R}_{nk}(\hcal):=\mathfrak{R}_{S_{\hcal},nk}(\hcal)$.
We now control $\ncal_\infty(\epsilon/G,\hcal,S_{\hcal})$ by Rademacher complexities of $\hcal$.
For any $\epsilon>2\mathfrak{R}_{nk}(\hcal)$, it follows from Part (b) of Lemma \ref{lem:shattering} that
\begin{equation}\label{cover-2}
\fat_\epsilon(\hcal)\leq \frac{4nk}{\epsilon^2}\mathfrak{R}^2_{S_{\hcal},nk}(\hcal)\leq nk.
\end{equation}
Note for any $f\in\fcal$, we have $f(\bx)^\top(f(\bx^+)-f(\bx^-))\in[-2R^2,2R^2]$.
It then follows from Eq. \eqref{shattering} and Eq. \eqref{cover-2} that (replace $B$ by $2R^2$)
\begin{align*}
\log\ncal_\infty(\epsilon,\hcal,S_{\hcal}) & \leq
1 +  \fat_{\epsilon/4}(\hcal)\log^2(32eR^4nk/\epsilon^2)\\
& \leq 1 +  \frac{64nk\mathfrak{R}^2_{nk}(\hcal)}{\epsilon^2}\log^2(32eR^4nk/\epsilon^2),\quad\epsilon\in(0,4R^2].
\end{align*}
We can combine the above inequality and Eq. \eqref{cover-1} to derive the following inequality for any $2G\mathfrak{R}_{nk}(\hcal)\leq\epsilon\leq 4GR^2$
\begin{equation}\label{cover-4}
\log\ncal_\infty(\epsilon,\gcal,S)\leq 1 + \frac{64nkG^2\mathfrak{R}^2_{nk}(\hcal)}{\epsilon^2}\log^2(32eR^4G^2nk/\epsilon^2).
\end{equation}
Let $\epsilon_N=24G\max\big\{\sqrt{k}\mathfrak{R}_{nk}(\hcal),n^{-\frac{1}{2}}\big\}$,
\[
\epsilon_j=2^{N-j}\epsilon_N,\quad j=0,\ldots,N-1,
\]
where
\[
N=\bigg\lceil\log_2\frac{2GR^2}{24G\max\Big\{\sqrt{k}\mathfrak{R}_{nk}(\hcal),n^{-\frac{1}{2}}\Big\}}\bigg\rceil.
\]
It is clear from the definition that
\[
\epsilon_0\geq 2GR^2\geq\epsilon_0/2.
\]
The Lipschitz continuity of $\ell$ implies
\[
\ell((\{h_f(\bx,\bx^+,\bx^-_i)\}_{i\in[k]}))-\ell((0,0,\ldots,0))\leq G\|h_f(\bx,\bx^+,\bx^-)\|_\infty  \leq 2R^2G.
\]
According to the above inequality and  Part (c) of Lemma \ref{lem:shattering}, we know (note  $\epsilon_N\geq2G\mathfrak{R}_{nk}(\hcal)$ and therefore Eq. \eqref{cover-4} holds for $\epsilon=\epsilon_j,j=1,\ldots,N$)
\begin{align*}
  & \mathfrak{R}_S(\gcal) \leq 2 \sum_{j=1}^{N}(\epsilon_j+\epsilon_{j-1})\sqrt{\frac{\log\ncal_\infty(\epsilon_j,\gcal,S)}{n}}+\epsilon_N \\
   & \leq 2n^{-\frac{1}{2}}\sum_{j=1}^{N}(\epsilon_j+\epsilon_{j-1})+
   \frac{16G\sqrt{nk}\mathfrak{R}_{nk}(\hcal)}{\sqrt{n}} \sum_{j=1}^{N}\frac{(\epsilon_j+\epsilon_{j-1})\log(32eR^4G^2nk/\epsilon_j^2)}{\epsilon_j}+\epsilon_N\\
   & \leq 6\epsilon_0n^{-\frac{1}{2}}+\epsilon_N+
   \frac{48G\sqrt{nk}\mathfrak{R}_{nk}(\hcal)}{\sqrt{n}} \sum_{j=1}^{N}\log(32eR^4G^2nk/\epsilon_j^2)\\
   & \leq 24GR^2n^{-\frac{1}{2}}+\epsilon_N+
   48GN\sqrt{k}\mathfrak{R}_{nk}(\hcal) \log(32eR^4G^2nk)+
   48G\sqrt{k}\mathfrak{R}_{nk}(\hcal) \sum_{j=1}^{N}\log(1/\epsilon_j^2).
\end{align*}
According to the definition of $\epsilon_k$, we know
\begin{align*}
  \sum_{j=1}^{N}\log(1/\epsilon_j^2)
  & = \sum_{j=1}^{N}\log(2^{2j}/\epsilon_0^2)
  = \sum_{j=1}^N\log(1/\epsilon_0^2) + \log 4\cdot \sum_{j=1}^Nj= N\log(1/\epsilon_0^2) + \frac{N(N+1)\log 4}{2}\\
  & =N\Big(\log1/\epsilon_0^2+(N+1)\log2\Big)=N\log2^{N+1}/\epsilon_0^2=N\log\Big(\frac{1}{\epsilon_N}\frac{2}{\epsilon_0}\Big)
  \leq N\log\Big(\frac{1}{\epsilon_N}\frac{2}{2GR^2}\Big)\\
  & \leq N\log\Big(\frac{\sqrt{n}}{24G}\frac{1}{GR^2}\Big)=N\log\frac{\sqrt{n}}{24G^2R^2}.
\end{align*}
We can combine the above two inequalities together to get
\begin{align*}
\mathfrak{R}_S(\gcal) &
 \leq 24GR^2n^{-\frac{1}{2}}+\epsilon_N + 48NG\sqrt{k}\mathfrak{R}_{nk}(\hcal) \Big(\log(32eR^4G^2nk)+\log\frac{\sqrt{n}}{24G^2R^2}\Big)\\
& \leq 24GR^2n^{-\frac{1}{2}}+\epsilon_N + 48G\sqrt{k}\mathfrak{R}_{nk}(\hcal) \Big(\log(32eR^4G^2nk)+\log\frac{\sqrt{n}}{24G^2R^2}\Big)\Big\lceil\log_2\frac{2GR^2\sqrt{n}}{24G}\Big\rceil\\
& \leq 24GR^2n^{-\frac{1}{2}}+\epsilon_N + 48G\sqrt{k}\mathfrak{R}_{nk}(\hcal) \Big(\log(4R^2n^{\frac{3}{2}}k)\Big)\Big\lceil\log_2\frac{R^2\sqrt{n}}{12}\Big\rceil,
\end{align*}
where we have used the definition of $N$ and $32e/24\leq 4$.
The proof is completed.
\end{proof}

\begin{proof}[Proof of Theorem \ref{thm:gen-unsupervised-loss}]
  Applying Lemma \ref{lem:gen-rademacher}, with probability at least $1-\delta$ the following inequality holds with probability at least $1-\delta$
  \[
  L_{un}(f)\leq \hat{L}_{un}(f)+2\mathfrak{R}_S(\gcal)+3B\sqrt{\frac{\log(2/\delta)}{2n}},\quad\forall f\in\fcal.
  \]
  According to Theorem \ref{thm:rademacher} and Lemma \ref{lem:rademacher-H}, we know
  \begin{align*}
    \mathfrak{R}_S(\gcal)
    & \leq 48G(R^2+1)n^{-\frac{1}{2}}+ 48G\sqrt{k}\mathfrak{R}_{nk}(\hcal)\bigg(1+ \log(R^2n^{\frac{3}{2}}k)\Big\lceil\log_2\frac{R^2\sqrt{n}}{12}\Big\rceil\bigg)\\
    & \leq 48G(R^2+1)n^{-\frac{1}{2}}+ \frac{48\sqrt{12}GR\sqrt{k}}{nk}\bigg(1+ \log(R^2n^{\frac{3}{2}}k)\Big\lceil\log_2\frac{R^2\sqrt{n}}{12}\Big\rceil\bigg)\mathfrak{C}.
  \end{align*}
  We can combine the above two inequalities together and derive the stated bound. The proof is completed.
\end{proof}

\section{Proof of Theorem \ref{thm:gen-self-bounding}\label{sec:proof-gen-self}}

To prove Theorem \ref{thm:gen-self-bounding}, we introduce the following lemma on generalization error bounds in terms of local Rademacher complexities~\citep{reeve2020optimistic}.

\begin{lemma}[\citealt{reeve2020optimistic}\label{lem:gen-local}]
  Consider a function class $\gcal$ of functions mapping $\zcal$ to $[0,b]$. For any $\widetilde{S}=\{\bz_i:i\in[n]\}$ and $g\in\gcal$, let $\hat{\ebb}_{\widetilde{S}}[g]=\frac{1}{n}\sum_{i\in[n]}g(\bz_i)$. Assume for any $\widetilde{S}\in\zcal^n$ and $r>0$, we have
  \[
  \mathfrak{R}_{\widetilde{S}}\big(\{g\in\gcal:\hat{\ebb}_{\widetilde{S}}[g]\leq r\}\big)\leq \phi_n(r),
  \]
  where $\phi_n:\rbb_+\mapsto\rbb_+$ is non-decreasing and $\phi_n(r)/\sqrt{r}$ is non-increasing. Let $\hat{r}_n$ be the largest solution of the equation $\phi_n(r)=r$. For any $\delta\in(0,1)$, with probability at least $1-\delta$ the following inequality holds uniformly for all $g\in\gcal$
  \[
  \ebb_{\bz}[g(\bz)]\leq \hat{\ebb}_{\widetilde{S}}[g]+90(\hat{r}_n+r_0)+4\sqrt{\hat{\ebb}_{\widetilde{S}}[g](\hat{r}_n+r_0)},
  \]
  where $r_0=b\big(\log(1/\delta)+6\log\log n\big)/n$.
\end{lemma}
\begin{proof}[Proof of Theorem \ref{thm:gen-self-bounding}]
For any $r>0$, we define $\fcal_r$ as a subset of $\fcal$ with the empirical error less than or equal to $r$
\[
\fcal_r=\Big\{f\in\fcal:\frac{1}{n}\sum_{j\in[n]}g_f(\bx_j,\bx_j^+,\bx_{j1}^-,\ldots,\bx_{jk}^-)\leq r\Big\}.
\]
Let
\[
 \Big\{\mathbf{r}^m = \big(r^m_{1,1},r^m_{1,2},\ldots,r^m_{1,k},\ldots,r^m_{n,1},r^m_{n,2},\ldots,r^m_{n,k}\big)\\
      :m\in[N]\Big\}
\]
be an $(\epsilon/(\sqrt{2r}G_s),\ell_\infty)$-cover of $\hcal_r:=\big\{h_f\in\hcal:f\in\fcal_r\big\}$ w.r.t. $S_{\hcal}$. Then, by the definition of $\ell_\infty$-cover we know for any $f\in\fcal_r$ we can find $m\in[N]$ such that
\[
\max_{j\in[n]}\max_{i\in[k]}\big|h_f(\bx_j,\bx_j^+,\bx_{ji}^-)-r^m_{j,i}\big|\leq\epsilon/(\sqrt{2r}G_s).
\]
According to the self-bounding Lipschitz continuity of $\ell$, we know
\begin{align*}
& \frac{1}{n}\sum_{j\in[n]}\big|\ell\big(\big\{f(\bx_j)^\top\big(f(\bx_j^+)-f(\bx_{ji}^-)\big)\big\}_{i=1}^k\big)-\ell\big(\{r^m_{j,i}\}_{i=1}^k\big)\big|^2\\
& \leq \frac{G_s^2}{n}\sum_{j\in[n]}\max\Big\{\ell\big(\big\{f(\bx_j)^\top\big(f(\bx_j^+)-f(\bx_{ji}^-)\big)\big\}_{i=1}^k\big),\ell\big(\{r^m_{j,i}\}_{i=1}^k\big)\Big\}\big\|\big(f(\bx_j)^\top\big(f(\bx_j^+)-f(\bx_{ji}^-)\big)\big)_{i=1}^k-\big(r^m_{j,i}\big)_{i=1}^k\big\|_\infty^2 \\
& \leq \frac{G_s^2}{n}\sum_{j\in[n]}\big(\ell\big(\big\{f(\bx_j)^\top\big(f(\bx_j^+)-f(\bx_{ji}^-)\big)\big\}_{i=1}^k\big)+\ell\big(\{r^m_{j,i}\}_{i=1}^k\big)\big)\big\|\big(h_f(\bx_j,\bx_j^+,\bx^-_{ji})\big)_{i=1}^k-\big(r^m_{j,i}\big)_{i=1}^k\big\|_\infty^2\\
&  \leq 2G_s^2r\epsilon^2/(2rG_s^2)=\epsilon^2,
\end{align*}
where we have used the following inequalities due to the definition of $\fcal_r$
\[
\frac{1}{n}\sum_{j\in[n]}\ell\big(\big\{f(\bx_j)^\top\big(f(\bx_j^+)-f(\bx_{ji}^-)\big)\big\}_{i=1}^k\big)\leq r,\qquad
\frac{1}{n}\sum_{j\in[n]}\ell\big(\{r^m_{j,i}\}_{i=1}^k\big)\leq r.
\]
Therefore, we have
\[
\ncal_2(\epsilon,\gcal_r,S)\leq\ncal_\infty(\epsilon/(\sqrt{2r}G_s),\hcal_r,S_{\hcal}),
\]
where $\gcal_r=\{g_f\in\gcal:f\in\fcal_r\}$.
Analyzing analogously to the proof of Theorem \ref{thm:rademacher}, we get (replacing $G$ there by $\sqrt{2r}G_s$)
\[
\mathfrak{R}_S(\gcal_r)\leq
24\sqrt{2r}G_s(R^2+1)n^{-\frac{1}{2}}+ 48\sqrt{2r}G_s\sqrt{k}\mathfrak{R}_{S_{\hcal},nk}(\hcal)\bigg(1+ \log(4R^2n^{\frac{3}{2}}k)\Big\lceil\log_2\frac{R^2\sqrt{n}}{12}\Big\rceil\bigg):=\psi_n(r).
\]
Let $\hat{r}_n$ be the point satisfying $\hat{r}_n=\psi_n(\hat{r}_n)$:
\[
\hat{r}_n=
24\sqrt{2\hat{r}_n}G_s(R^2+1)n^{-\frac{1}{2}}+ 48\sqrt{2\hat{r}_n}G_s\sqrt{k}\mathfrak{R}_{S_{\hcal},nk}(\hcal)\bigg(1+ \log(4R^2n^{\frac{3}{2}}k)\Big\lceil\log_2\frac{R^2\sqrt{n}}{12}\Big\rceil\bigg),
\]
from which we get
\begin{align*}
\hat{r}_n 
& = \widetilde{O}\Big(G_s^2R^4n^{-1}+G_s^2k\mathfrak{R}^2_{S_{\hcal},nk}(\hcal)\Big).
\end{align*}
We can apply Lemma \ref{lem:gen-local} to get the following inequality with probability at least $1-\delta$ uniformly for all $f\in\fcal$
\[
L_{un}(f)= \hat{L}_{un}(f)+\widetilde{O}\Big(Bn^{-1}+G_s^2R^4n^{-1}+G_s^2k\mathfrak{R}^2_{S_{\hcal},nk}(\hcal)\Big)
+\widetilde{O}\Big(\sqrt{B}n^{-\frac{1}{2}}+G_sR^2n^{-\frac{1}{2}}+G_s\sqrt{k}\mathfrak{R}_{S_{\hcal},nk}(\hcal)\Big)\hat{L}^{\frac{1}{2}}_{un}(f).
\]
We can apply Lemma \ref{lem:rademacher-H} to control $\mathfrak{R}_{S_{\hcal},nk}(\hcal)$ and derive the following bound
\[
L_{un}(f)= \hat{L}_{un}(f)+\widetilde{O}\Big(Bn^{-1}+G_s^2R^4n^{-1}+G_s^2R^2n^{-2}k^{-1}\mathfrak{C}^2\Big)
+\widetilde{O}\Big(\sqrt{B}n^{-\frac{1}{2}}+G_sR^2n^{-\frac{1}{2}}+G_sRn^{-1}k^{-\frac{1}{2}}\mathfrak{C}\Big)\hat{L}^{\frac{1}{2}}_{un}(f).
\]
The proof is completed.
\end{proof}

\section{Proof on Rademacher Complexities\label{sec:proof-rade}}
We first prove  Rademacher complexity bounds for feature spaces in Lemma \ref{lem:rademacher-feature-space}, and then give lower bounds. Finally, we will apply it to prove Proposition \ref{prop:rade-linear} and Proposition \ref{prop:rade-dnn}.

\begin{proof}[Proof of Lemma \ref{lem:rademacher-feature-space}]
Let $U=(\bu_1,\ldots,\bu_d)^\top$ and $V_S=(\bv(\bx_1),\ldots,\bv(\bx_n))$. Then it is clear
\[
UV_S=\begin{pmatrix}
       \bu_1^\top\bv(\bx_1) & \cdots & \bu_1^\top\bv(\bx_n) \\
       \vdots & \vdots & \vdots \\
       \bu_d^\top\bv(\bx_1) & \cdots & \bu_d^\top\bv(\bx_n)
     \end{pmatrix}.
\]
Let
\[
M_\epsilon=\begin{pmatrix}
             \epsilon_{1,1} & \cdots & \epsilon_{1,n} \\
             \vdots & \vdots & \vdots \\
             \epsilon_{d,1} & \cdots & \epsilon_{d,n} \\
           \end{pmatrix}\in\rbb^{d\times n}.
\]
Then we have
\begin{align*}
  \sum_{t\in[d]}\sum_{j\in[n]}\epsilon_{t,j}\bu_t^\top\bv(\bx_j)&=\big\langle M_\epsilon, UV_S\rangle=\text{trace}(M_\epsilon^\top UV_S)=\text{trace}(UV_SM_\epsilon^\top)\\
  & = \big\langle U^\top,V_SM_\epsilon^\top\big\rangle \leq \|U^\top\|\|V_SM_\epsilon^\top\|_*,
\end{align*}
where $\text{trace}$ denotes the trace of a matrix.
Therefore, we have
\begin{align*}
  \sup_{f\in\fcal}\sum_{t\in[d]}\sum_{j\in[n]}\epsilon_{t,j}f_t(\bx_j) & =
  \sup_{U\in\ucal,\bv\in\vcal}\sum_{t\in[d]}\sum_{j\in[n]}\epsilon_{t,j}\bu_t^\top \bv(\bx_j) \\
  & \leq \Lambda \sup_{\bv\in\vcal}\|V_SM_\epsilon^\top\|_*
   = \Lambda\sup_{\bv\in\vcal}\big\|\big(\sum_{j\in[n]}\epsilon_{1,j}\bv(\bx_j),\ldots,\sum_{j\in[n]}\epsilon_{d,j}\bv(\bx_j)\big)\big\|_*.
\end{align*}
The proof is completed.
\end{proof}

\begin{proof}[Proof of Lemma \ref{lem:lower}]
Note
\begin{multline*}
\Big|\sum_{j\in[nk]}\sum_{t\in[d]}\big(\epsilon_{j,t,1}f_t(\tilde{\bx}_j)+\epsilon_{j,t,2}f_t(\tilde{\bx}_j^+)+\epsilon_{j,t,3}f_t(\tilde{\bx}_j^-)\big)\Big|=\\
\max\Big\{
\sum_{j\in[nk]}\sum_{t\in[d]}\big(\epsilon_{j,t,1}f_t(\tilde{\bx}_j)+\epsilon_{j,t,2}f_t(\tilde{\bx}_j^+)+\epsilon_{j,t,3}f_t(\tilde{\bx}_j^-)\big),
-\sum_{j\in[nk]}\sum_{t\in[d]}\big(\epsilon_{j,t,1}f_t(\tilde{\bx}_j)+\epsilon_{j,t,2}f_t(\tilde{\bx}_j^+)+\epsilon_{j,t,3}f_t(\tilde{\bx}_j^-)\big)
\Big\}.
\end{multline*}
  According to the symmetry of $\fcal$ we know
  \begin{align*}
    \mathfrak{C}&=\max_{\{(\tilde{\bx}_j,\tilde{\bx}_j^+,\tilde{\bx}_j^-)\}_{j=1}^{nk}\subseteq S_{\hcal}}\ebb_{\bm{\epsilon}\sim\{\pm1\}^{nk}\times\{\pm1\}^d\times\{\pm1\}^3}
  \Big[\sup_{f\in\fcal}\Big|\sum_{j\in[nk]}\sum_{t\in[d]}\big(\epsilon_{j,t,1}f_t(\tilde{\bx}_j)+\epsilon_{j,t,2}f_t(\tilde{\bx}_j^+)+\epsilon_{j,t,3}f_t(\tilde{\bx}_j^-)\big)\Big|\Big]\\
  & \geq
  \sup_{f\in\fcal}\max_{\{(\tilde{\bx}_j,\tilde{\bx}_j^+,\tilde{\bx}_j^-)\}_{j=1}^{nk}\subseteq S_{\hcal}}\ebb_{\bm{\epsilon}\sim\{\pm1\}^{nk}\times\{\pm1\}^d\times\{\pm1\}^3}
  \Big[\Big|\sum_{j\in[nk]}\sum_{t\in[d]}\big(\epsilon_{j,t,1}f_t(\tilde{\bx}_j)+\epsilon_{j,t,2}f_t(\tilde{\bx}_j^+)+\epsilon_{j,t,3}f_t(\tilde{\bx}_j^-)\big)\Big|\Big],
  \end{align*}
  where we have used the Jensen's inequality in the last step.

  Since we take maximization over $\{(\tilde{\bx}_j,\tilde{\bx}_j^+,\tilde{\bx}_j^-)\}_{j=1}^{nk}\subseteq S_{\hcal}$, we can choose $(\tilde{\bx}_j,\tilde{\bx}_j^+,\tilde{\bx}_j^-)=(\tilde{\bx},\tilde{\bx}^+,\tilde{\bx}^-)$ for any $(\tilde{\bx},\tilde{\bx}^+,\tilde{\bx}^-)\in S_{\hcal}$. Then we get
  \begin{align*}
    \mathfrak{C} & \geq
    \sup_{f\in\fcal}\max_{(\tilde{\bx},\tilde{\bx}^+,\tilde{\bx}^-)\in  S_{\hcal}}\ebb_{\bm{\epsilon}\sim\{\pm1\}^{nk}\times\{\pm1\}^d\times\{\pm1\}^3}
  \Big[\Big|\sum_{j\in[nk]}\sum_{t\in[d]}\big(\epsilon_{j,t,1}f_t(\tilde{\bx})+\epsilon_{j,t,2}f_t(\tilde{\bx}^+)+\epsilon_{j,t,3}f_t(\tilde{\bx}^-)\big)\Big|\Big]\\
     & \geq
     2^{-\frac{1}{2}}\sup_{f\in\fcal}\max_{(\tilde{\bx},\tilde{\bx}^+,\tilde{\bx}^-)\in  S_{\hcal}}
  \Big(\sum_{j\in[nk]}\sum_{t\in[d]}\big(f_t^2(\tilde{\bx})+f_t^2(\tilde{\bx}^+)+f_t^2(\tilde{\bx}^-)\big)\Big)^{\frac{1}{2}}\\
  & = 2^{-\frac{1}{2}}\sup_{f\in\fcal}\max_{(\tilde{\bx},\tilde{\bx}^+,\tilde{\bx}^-)\in  S_{\hcal}}
  \Big(\sum_{j\in[nk]}\big(\|f(\tilde{\bx})\|_2^2+\|f(\tilde{\bx}^+)\|_2^2+\|f(\tilde{\bx}^-)\|_2^2\big)\Big)^{\frac{1}{2}}\\
  & = \sqrt{2^{-1}nk}\sup_{f\in\fcal}\max_{(\tilde{\bx},\tilde{\bx}^+,\tilde{\bx}^-)\in  S_{\hcal}}
  \Big(\|f(\tilde{\bx})\|_2^2+\|f(\tilde{\bx}^+)\|_2^2+\|f(\tilde{\bx}^-)\|_2^2\Big)^{\frac{1}{2}},
  \end{align*}
  where we have used the following Khitchine-Kahane inequality \citep{de2012decoupling}
  \begin{equation}\label{kkahane}
    \ebb_\epsilon\big|\sum_{i=1}^{n}\epsilon_it_i\big|\geq 2^{-\frac{1}{2}}\big[\sum_{i=1}^{n}|t_i|^2\big]^{\frac{1}{2}},\quad\forall t_1,\ldots,t_n\in\rbb,
  \end{equation}
  The proof is completed.
\end{proof}

\begin{remark}
  The analysis in the proof implies a lower bound for $\mathfrak{R}_{\widetilde{S}}(\widetilde{\fcal})$ for a symmetric $\widetilde{\fcal}$ and $\widetilde{S}=\{\bz_1,\ldots,\bz_n\}$
  \[
  \mathfrak{R}_{\widetilde{S}}(\widetilde{\fcal})\geq
  \frac{1}{\sqrt{2}n}\sup_{f\in \widetilde{\fcal}}\big\|\big(f(\bz_1),\ldots,f(\bz_n)\big)\big\|_2.
  \]
  Indeed, by the symmetry of $\fcal$, the Jensen inequality and Eq. \eqref{kkahane}, we have
  \begin{align*}
    \mathfrak{R}_{\widetilde{S}}(\widetilde{\fcal})&=\ebb_{\bm{\epsilon}}\big[\sup_{f\in \widetilde{\fcal}}\frac{1}{n}\sum_{i\in[n]}\epsilon_if(\bz_i)\big]
    =\ebb_{\bm{\epsilon}}\big[\sup_{f\in \widetilde{\fcal}}\frac{1}{n}\big|\sum_{i\in[n]}\epsilon_if(\bz_i)\big|\big]\\
    & \geq \frac{1}{n}\sup_{f\in \widetilde{\fcal}}\ebb_{\bm{\epsilon}}\big[\big|\sum_{i\in[n]}\epsilon_if(\bz_i)\big|\big]
     \geq \frac{1}{\sqrt{2}n}\sup_{f\in \widetilde{\fcal}}\Big(\sum_{i\in[n]}f^2(\bz_i)\Big)^{\frac{1}{2}}.
  \end{align*}
\end{remark}

\subsection{Proof of Proposition \ref{prop:rade-linear}\label{sec:proof-rade-linear}}

The following Khintchine-Kahane inequality~\citep{de2012decoupling,lust1991non} is very useful for us to estimate Rademacher complexities.
\begin{lemma}\label{lem:khitchine-kahane}
Let $\epsilon_1,\ldots,\epsilon_n$ be a sequence of independent Rademacher variables.
\begin{enumerate}[(a)]
  \item Let $\bv_1,\ldots,\bv_n\in\hcal$, where $\hcal$ is a Hilbert space with $\|\cdot\|$ being the associated norm. Then, for any $p\geq1$ there holds
  \begin{equation}
    \big[\ebb_\epsilon\|\sum_{i=1}^{n}\epsilon_i\bv_i\|^p\big]^{\frac{1}{p}} \leq \max(\sqrt{p-1},1)\big[\sum_{i=1}^{n}\|\bv_i\|^2\big]^{\frac{1}{2}}.\label{khitchine-kahane}
  \end{equation}
  \item Let $X_1,\ldots,X_n$ be a set of matrices of the same dimension. For all $q\geq2$,
  \begin{equation}\label{khitchine-kahane-matrix}
    \Big(\ebb_{\bm{\epsilon}}\big\|\sum_{i=1}^{n}\epsilon_iX_i\big\|_{S_q}^q\Big)^{\frac{1}{q}}\leq 2^{-\frac{1}{4}}\sqrt{\frac{q\pi}{e}}\max\Big\{
    \big\|\big(\sum_{i=1}^{n}X_i^\top X_i\big)^{\frac{1}{2}}\big\|_{S_q},\big\|\big(\sum_{i=1}^{n}X_i X_i^\top\big)^{\frac{1}{2}}\big\|_{S_q}\Big\}.
  \end{equation}
\end{enumerate}
\end{lemma}
\begin{proof}[Proof of Proposition \ref{prop:loss-lip}]
  Let $q\geq p$. It is clear $q^*\leq p^*$.
  The dual norm of $\|\cdot\|_{2,p}$ is $\|\cdot\|_{2,p^*}$. Therefore, according to Lemma \ref{lem:rademacher-feature-space} and $\|\cdot\|_{p^*}\leq\|\cdot\|_{q^*}$ we know
  \begin{align*}
  \ebb_{\bm{\epsilon}} \sup_{f\in\fcal}\sum_{t\in[d]}\sum_{j\in[n]}\epsilon_{t,j}f_t(\bx_j)
  &\leq \Lambda\ebb_{\bm{\epsilon}}\Big(\sum_{t\in[d]}\big\|\sum_{j\in[n]}\epsilon_{t,j}\bx_j\big\|_2^{p^*}\Big)^{1/p^*}
  \leq \Lambda\ebb_{\bm{\epsilon}}\Big(\sum_{t\in[d]}\big\|\sum_{j\in[n]}\epsilon_{t,j}\bx_j\big\|_2^{q^*}\Big)^{1/q^*}\\
  & \leq
  \Lambda\Big(\ebb_{\bm{\epsilon}}\Big[\sum_{t\in[d]}\big\|\sum_{j\in[n]}\epsilon_{t,j}\bx_j\big\|_2^{q^*}\Big]\Big)^{1/q^*}
  = \Lambda\Big(d\ebb_{\bm{\epsilon}}\big\|\sum_{j\in[n]}\epsilon_{j}\bx_j\big\|_2^{q^*}\Big)^{1/q^*},
  \end{align*}
  where we have used Jense's inequality and the concavity of $x\mapsto x^{1/q^*}$.
  By Lemma \ref{lem:khitchine-kahane}, we know
  \[
  \ebb_{\bm{\epsilon}}\big\|\sum_{j\in[n]}\epsilon_{j}\bx_j\big\|_2^{q^*}\leq
  \max(\sqrt{q^*-1},1)^{q^*}\Big(\sum_{j\in[n]}\|\bx_j\|_2^2\Big)^{\frac{q^*}{2}}.
  \]
  It then follows that
  \[
  \ebb_{\bm{\epsilon}} \sup_{f\in\fcal}\sum_{t\in[d]}\sum_{j\in[n]}\epsilon_{t,j}f_t(\bx_j)
  \leq \Lambda d^{1/q^*}\max(\sqrt{q^*-1},1)\Big(\sum_{j\in[n]}\|\bx_j\|_2^2\Big)^{\frac{1}{2}}.
  \]
  Note the above inequality holds for any $q\geq p$. This proves Part (a).

  We now prove Part (b). Since the dual norm of $\|\cdot\|_{S_p}$ is $\|\cdot\|_{S_{p^*}}$, by Lemma \ref{lem:rademacher-feature-space} we know
  \begin{align*}
  \ebb_{\bm{\epsilon}} \sup_{f\in\fcal}\sum_{t\in[d]}\sum_{j\in[n]}\epsilon_{t,j}f_t(\bx_j)
  & \leq \Lambda \ebb_{\bm{\epsilon}}\big\|\big(\sum_{j\in[n]}\epsilon_{1,j}\bx_j,\ldots,\sum_{j\in[n]}\epsilon_{d,j}\bx_j\big)\big\|_{S_{p^*}}.
  \end{align*}
  For any $t\in[d]$ and $j\in[n]$, define
  \[
  \widetilde{X}_{t,j}=\begin{pmatrix}
                         0 & \cdots & 0 & \bx_j & 0 & \cdots & 0
                      \end{pmatrix},
  \]
  i.e., the $t$-th column of $\widetilde{X}_{t,j}=\bx_j$, and other columns are zero vectors. This implies that
  \[
  \big(\sum_{j\in[n]}\epsilon_{1,j}\bx_j,\ldots,\sum_{j\in[n]}\epsilon_{d,j}\bx_j\big)
  =\sum_{t\in[d]}\sum_{j\in[n]}\epsilon_{t,j}\widetilde{X}_{t,j}.
  \]
  It is clear that $\widetilde{X}_{t,j}\widetilde{X}_{t,j}^\top = \bx_j\bx_j^\top$ and
  \[
   \widetilde{X}_{t,j}^\top \widetilde{X}_{t,j}=
  \begin{pmatrix}
    0 & \cdots & \cdots & 0 \\
    \vdots & \dots & \cdots & 0 \\
    \vdots & 0 & \bx_j^\top\bx_j & 0 \\
    0 & \cdots & \cdots & \dots.
  \end{pmatrix}=\bx_j^\top\bx_j\text{diag}(\underbrace{0,\ldots,0}_{t-1},1,0\ldots,0),
  \]
  where $\text{diag}(a_1,\ldots,a_n)$ denotes the diagonal matrix with elements $a_1,\ldots,a_n$.
  Therefore, we have
  \[
  \sum_{t\in[d]}\sum_{j\in[n]}\widetilde{X}_{t,j}\widetilde{X}_{t,j}^\top = d\sum_{j\in[n]}\bx_j\bx_j^\top
  \]
  and
  \[
  \sum_{t\in[d]}\sum_{j\in[n]}\widetilde{X}_{t,j}^\top\widetilde{X}_{t,j} = \big(\sum_{j\in[n]}\bx_j^\top\bx_j\big)\ibb_{d\times d},
  \]
  where $\ibb_{d\times d}$ denotes the identity matrix in $\rbb^{d\times d}$.
  Therefore, we can apply Lemma \ref{lem:khitchine-kahane} to show that
  \begin{align*}
  \ebb_{\bm{\epsilon}} \sup_{f\in\fcal}\sum_{t\in[d]}\sum_{j\in[n]}\epsilon_{t,j}f_t(\bx_j)
  & \leq \Lambda \Big(\ebb_{\bm{\epsilon}}\big\|\big(\sum_{j\in[n]}\epsilon_{1,j}\bx_j,\ldots,\sum_{j\in[n]}\epsilon_{d,j}\bx_j\big)\big\|_{S_q^*}^{q^*}\Big)^{1/q^*}\\
  & \leq \Lambda2^{-\frac{1}{4}}\sqrt{\frac{q^*\pi}{e}}\max\Big\{\Big\|\Big(d\sum_{j\in[n]}\bx_j\bx_j^\top\Big)^{\frac{1}{2}}\Big\|_{S_{q^*}},d^{1/q^*}\big(\sum_{j\in[d]}\|\bx_j\|_2^2\big)^{\frac{1}{2}}\Big\}.
  \end{align*}
  The proof is completed.
\end{proof}

\subsection{Proof of Proposition \ref{prop:rade-dnn}\label{sec:proof-rade-dnn}}

For convenience we introduce the following sequence of function spaces
\[
\vcal_k=\Big\{\bx\mapsto \sigma_{k}\big(V_{k}\sigma\big(V_{k-1}\cdots \sigma(V_1\bx)\big)\big):\|V_j\|_F\leq B_j\Big\},\quad k\in[L].
\]
To prove Proposition \ref{prop:rade-dnn}, we need to introduce several lemmas.
The following lemma shows how the supremum over a matrix can be transferred to a supremum over a vector. It is an extension of Lemma 1 in \citet{golowich2018size} from $d=1$ to $d\in\nbb$, and can be proved exactly by the arguments in \citet{golowich2018size}.

\begin{lemma}\label{lem:5-mat}
  Let $\sigma:\rbb\mapsto\rbb$ be a $1$-Lipschitz continuous, positive-homogeneous activation function which is applied elementwise. Then for any vector-valued function class $\widetilde{\fcal}$
  \[
  \sup_{\tilde{f}\in\widetilde{\fcal},V\in\rbb^{h\times h'}:\|V\|_F\leq B}\sum_{t\in[d]}\Big\|\sum_{j\in[n]}\epsilon_{t,j}\sigma(V\tilde{f}(\bx_j))\Big\|_2^2\leq B^2\sup_{\tilde{f}\in\widetilde{\fcal},\tilde{\bv}\in\rbb^{h'}:\|\tilde{\bv}\|_2\leq 1}\sup_{\|\tilde{\bv}\|_2\leq 1}\sum_{t\in[d]}\Big|\sum_{j\in[n]}\epsilon_{t,j}\sigma(\tilde{\bv}^\top \tilde{f}(\bx_j))\Big|^2.
  \]
\end{lemma}
\begin{proof}
  Let $\bv_1^\top,\ldots,\bv_h^\top$ be rows of matrix $V$, i.e., $V^\top=\big(\bv_1,\ldots,\bv_h\big)$. Then by the positive-homogeneous property of activation function we have
  \begin{align*}
    \sum_{t\in[d]}\Big\|\sum_{j\in[n]}\epsilon_{t,j}\sigma(V\tilde{f}(x_i))\Big\|_2^2 & = \sum_{t\in[d]}\left\|\begin{pmatrix}
                                                                  \sum_{j\in[n]}\epsilon_{t,j}\sigma(\bv_1^\top \tilde{f}(\bx_j)) \\
                                                                  \vdots \\
                                                                  \sum_{j\in[n]}\epsilon_{t,j}\sigma(\bv_h^\top \tilde{f}(\bx_j)) \\
                                                                \end{pmatrix}\right\|_2^2 =\sum_{t\in[d]}\sum_{r\in [h]}\left(\sum_{j\in[n]}\epsilon_{t,j}\sigma(\bv_r^\top \tilde{f}(\bx_j))\right)^2\\
     & =\sum_{r\in[h]}\|\bv_r\|_2^2\sum_{t\in[d]}\bigg(\sum_{j\in[n]}\epsilon_{t,j}\sigma\Big(\frac{\bv_r^\top}{\|\bv_r\|_2}\tilde{f}(\bx_j)\Big)\bigg)^2 \\
     & \leq \Big(\sum_{r\in[h]} \|\bv_r\|_2^2\Big)\max_{r\in[h]}\sum_{t\in[d]}\Big|\sum_{j\in[n]}\epsilon_{t,j}\sigma\Big(\frac{\bv_r^\top}{\|\bv_r\|_2}\tilde{f}(\bx_j)\Big)\Big|^2\\
     & \leq B^2\sup_{\|\tilde{\bv}\|_2\leq 1}\sum_{t\in[d]}\Big|\sum_{j\in[n]}\epsilon_{t,j}\sigma(\tilde{\bv}^\top \tilde{f}(\bx_j))\Big|^2.
  \end{align*}
  The proof is completed.
\end{proof}

The following lemma gives a general contraction lemma for Rademacher complexities. It allows us to remove a nonlinear function $\psi$, which is very useful for us to handle the activation function in DNNs.
\begin{lemma}[Contraction Lemma, Thm 11.6 in \citet{boucheron2013concentration}\label{lem:rade-cont}]
  Let $\tilde{\tau}:\rbb_+\mapsto\rbb_+$ be convex and nondecreasing. Suppose $\psi:\rbb\mapsto\rbb$ is contractive in the sense $|\psi(t)-\psi(\tilde{t})|\leq |t-\tilde{t}|$ and $\psi(0)=0$. Then the following inequality holds for any $\widetilde{\fcal}$
  \[
  \ebb_{\bm{\epsilon}}\tilde{\tau}\bigg(\sup_{f\in\widetilde{\fcal}}\sum_{i=1}^{n}\epsilon_i\psi\big(f(x_i)\big)\bigg)\leq \ebb_{\bm{\epsilon}}\tilde{\tau}\bigg(\sup_{f\in\widetilde{\fcal}}\sum_{i=1}^{n}\epsilon_if(x_i)\bigg).
  \]
\end{lemma}
The following lemma gives bounds of moment generation functions for a random variable $Z=\sum_{1\leq i<j\leq n}\epsilon_{i}\epsilon_ja_{ij}$, which is called a Rademacher chaos variable~ \citep{de2012decoupling,ying2010rademacher}.
\begin{lemma}[page 167 in \citet{de2012decoupling}\label{lem:chaos}]
  Let $\epsilon_i,i\in[n]$ be independent Rademacher variables. Let $a_{i,j}\in\rbb,i,j\in[n]$.
  Then for $Z=\sum_{1\leq i<j\leq n}\epsilon_{i}\epsilon_ja_{ij}$ we have
  \[
  \ebb_{\bm{\epsilon}}\exp\Big(|Z|/(4es)\Big)\leq 2,\quad\text{where } s^2:=\sum_{1\leq i<j\leq n}a_{i,j}^2.
  \]
\end{lemma}

\begin{proof}[Proof of Proposition \ref{prop:rade-dnn}]
  The dual norm of $\|\cdot\|_F$ is $\|\cdot\|_F$. Therefore, according to Lemma \ref{lem:rademacher-feature-space} we know
  \begin{align*}
  \ebb_{\bm{\epsilon}} \sup_{f\in\fcal}\sum_{t\in[d]}\sum_{j\in[n]}\epsilon_{t,j}f_t(\bx_j)
  &\leq \Lambda\ebb_{\bm{\epsilon}}\sup_{\bv\in\vcal}\Big(\sum_{t\in[d]}\big\|\sum_{j\in[n]}\epsilon_{t,j}\bv(\bx_j)\big\|_2^{2}\Big)^{1/2}\\
  & = \Lambda\ebb_{\bm{\epsilon}}\bigg(\sup_{\tilde{f}\in\vcal_{L-1},V:\|V\|_F\leq B_L}\sum_{t\in[d]}\Big\|\sum_{j\in[n]}\epsilon_{t,j}\sigma(V\tilde{f}(\bx_j))\Big\|_2^2\bigg)^{\frac{1}{2}}\\
  & \leq \Lambda B_L\ebb_{\bm{\epsilon}}\bigg(\sup_{\tilde{f}\in\vcal_{L-1},\tilde{\bv}:\|\tilde{\bv}\|_2\leq 1}\sum_{t\in[d]}\Big|\sum_{j\in[n]}\epsilon_{t,j}\sigma(\tilde{\bv}^\top \tilde{f}(\bx_j))\Big|^2\bigg)^{\frac{1}{2}},
  \end{align*}
  where we have used Lemma \ref{lem:5-mat} in the second inequality.
  Let $\lambda\geq0$ and $\tau(x)=\exp(\lambda x^2)$. It is clear that $\tau$ is convex and increasing in the interval $[0,\infty)$.
  It then follows from the Jensen's inequality that
  \begin{align*}
     \exp\bigg(\lambda\Big(\ebb_{\bm{\epsilon}} \sup_{f\in\fcal}\sum_{t\in[d]}\sum_{j\in[n]}\epsilon_{t,j}f_t(\bx_j)\Big)^2\bigg)
     & \leq \exp\bigg(\lambda\bigg(\Lambda B_L\ebb_{\bm{\epsilon}}\bigg(\sup_{\tilde{f}\in\vcal_{L-1},\tilde{\bv}:\|\tilde{\bv}\|_2\leq 1}\sum_{t\in[d]}\Big|\sum_{j\in[n]}\epsilon_{t,j}\sigma(\tilde{\bv}^\top \tilde{f}(\bx_j))\Big|^2\bigg)^{\frac{1}{2}}\bigg)^2\bigg) \\
     & \leq \ebb_{\bm{\epsilon}}\exp\bigg(\lambda\bigg(\Lambda B_L\bigg(\sup_{\tilde{f}\in\vcal_{L-1},\tilde{\bv}:\|\tilde{\bv}\|_2\leq 1}\sum_{t\in[d]}\Big|\sum_{j\in[n]}\epsilon_{t,j}\sigma(\tilde{\bv}^\top \tilde{f}(\bx_j))\Big|^2\bigg)^{\frac{1}{2}}\bigg)^2\bigg) \\
     & = \ebb_{\bm{\epsilon}}\exp\bigg(\lambda\Lambda^2B_L^2\sup_{\tilde{f}\in\vcal_{L-1},\tilde{\bv}:\|\tilde{\bv}\|_2\leq 1}\sum_{t\in[d]}\Big|\sum_{j\in[n]}\epsilon_{t,j}\sigma(\tilde{\bv}^\top \tilde{f}(\bx_j))\Big|^2\bigg)\\
     & \leq \ebb_{\bm{\epsilon}}\exp\bigg(\lambda\Lambda^2B_L^2\sum_{t\in[d]}\sup_{\tilde{f}\in\vcal_{L-1},\tilde{\bv}:\|\tilde{\bv}\|_2\leq 1}\Big|\sum_{j\in[n]}\epsilon_{t,j}\sigma(\tilde{\bv}^\top \tilde{f}(\bx_j))\Big|^2\bigg)\\
     & =  \ebb_{\bm{\epsilon}}\prod_{t\in[d]}\exp\bigg(\lambda\Lambda^2B_L^2\sup_{\tilde{f}\in\vcal_{L-1},\tilde{\bv}:\|\tilde{\bv}\|_2\leq 1}\Big|\sum_{j\in[n]}\epsilon_{t,j}\sigma(\tilde{\bv}^\top \tilde{f}(\bx_j))\Big|^2\bigg) \\
     & =  \prod_{t\in[d]}\ebb_{\bm{\epsilon}_t}\exp\bigg(\lambda\Lambda^2B_L^2\sup_{\tilde{f}\in\vcal_{L-1},\tilde{\bv}:\|\tilde{\bv}\|_2\leq 1}\Big|\sum_{j\in[n]}\epsilon_{t,j}\sigma(\tilde{\bv}^\top \tilde{f}(\bx_j))\Big|^2\bigg) \\
     & = \ebb_{\bm{\epsilon}\sim\{\pm1\}^{n}}\exp\bigg(d\lambda\Lambda^2B_L^2\sup_{\tilde{f}\in\vcal_{L-1},\tilde{\bv}:\|\tilde{\bv}\|_2\leq 1}\Big|\sum_{j\in[n]}\epsilon_{j}\sigma(\tilde{\bv}^\top \tilde{f}(\bx_j))\Big|^2\bigg),
  \end{align*}
  where we have used the independency between $\epsilon_{t}=(\epsilon_{t,j})_{j\in[n]},t\in[d]$. Let $\tilde{\tau}:\rbb_+\mapsto\rbb_+$ be defined as $\tilde{\tau}(x)=\exp(d\lambda\Lambda^2B_L^2x^2)$. Then we have
  \begin{align}
    & \exp\bigg(\lambda\Big(\ebb_{\bm{\epsilon}} \sup_{f\in\fcal}\sum_{t\in[d]}\sum_{j\in[n]}\epsilon_{t,j}f_t(\bx_j)\Big)^2\bigg)  \leq
    \ebb_{\bm{\epsilon}\sim\{\pm1\}^{n}}\tilde{\tau}\bigg(\sup_{\tilde{f}\in\vcal_{L-1},\tilde{\bv}:\|\tilde{\bv}\|_2\leq 1}\Big|\sum_{j\in[n]}\epsilon_{j}\sigma(\tilde{\bv}^\top \tilde{f}(\bx_j))\Big|\bigg)\notag\\
    & \leq \ebb_{\bm{\epsilon}\sim\{\pm1\}^{n}}\tilde{\tau}\bigg(\sup_{\tilde{f}\in\vcal_{L-1},\tilde{\bv}:\|\tilde{\bv}\|_2\leq 1}\sum_{j\in[n]}\epsilon_{j}\sigma(\tilde{\bv}^\top \tilde{f}(\bx_j))\bigg)
    + \ebb_{\bm{\epsilon}\sim\{\pm1\}^{n}}\tilde{\tau}\bigg(\sup_{\tilde{f}\in\vcal_{L-1},\tilde{\bv}:\|\tilde{\bv}\|_2\leq 1}-\sum_{j\in[n]}\epsilon_{j}\sigma(\tilde{\bv}^\top \tilde{f}(\bx_j))\bigg)\notag\\
    & = 2 \ebb_{\bm{\epsilon}\sim\{\pm1\}^{n}}\tilde{\tau}\bigg(\sup_{\tilde{f}\in\vcal_{L-1},\tilde{\bv}:\|\tilde{\bv}\|_2\leq 1}\sum_{j\in[n]}\epsilon_{j}\sigma(\tilde{\bv}^\top \tilde{f}(\bx_j))\bigg)
    \leq 2 \ebb_{\bm{\epsilon}\sim\{\pm1\}^{n}}\tilde{\tau}\bigg(\sup_{\tilde{f}\in\vcal_{L-1},\tilde{\bv}:\|\tilde{\bv}\|_2\leq 1}\sum_{j\in[n]}\epsilon_{j}\tilde{\bv}^\top \tilde{f}(\bx_j)\bigg)\notag\\
    & = \ebb_{\bm{\epsilon}\sim\{\pm1\}^{n}}\tilde{\tau}\bigg(\sup_{\tilde{f}\in\vcal_{L-1},\tilde{\bv}:\|\tilde{\bv}\|_2\leq 1}\tilde{\bv}^\top \sum_{j\in[n]}\epsilon_{j}\tilde{f}(\bx_j)\bigg)
    = \ebb_{\bm{\epsilon}\sim\{\pm1\}^{n}}\tilde{\tau}\bigg(\sup_{\tilde{f}\in\vcal_{L-1}}\Big\|\sum_{j\in[n]}\epsilon_{j}\tilde{f}(\bx_j)\Big\|_2\bigg).\label{rade-dnn-1}
  \end{align}
  where we have used Lemma \ref{lem:rade-cont} and the contraction property of $\sigma$ in the last inequality.

  According to Lemma \ref{lem:5-mat}, we know
  \begin{align*}
    \ebb_{\bm{\epsilon}\sim\{\pm1\}^{n}}\tilde{\tau}\bigg(\sup_{\tilde{f}\in\vcal_{L-1}}\Big\|\sum_{j\in[n]}\epsilon_{j}\tilde{f}(\bx_j)\Big\|_2\bigg) & =
    \ebb_{\bm{\epsilon}\sim\{\pm1\}^{n}}\tilde{\tau}\bigg(\sup_{\|V_{L-1}\|_F\leq B_{L-1},\tilde{f}\in\vcal_{L-2}}\Big\|\sum_{j\in[n]}\epsilon_{j}\sigma\big(V_{L-1}\tilde{f}(\bx_j)\big)\Big\|_2\bigg) \notag \\
     & \leq \ebb_{\bm{\epsilon}\sim\{\pm1\}^{n}}\tilde{\tau}\bigg(B_{L-1}\sup_{\|\tilde{\bv}\|_2\leq 1,\tilde{f}\in\vcal_{L-2}}\Big|\sum_{j\in[n]}\epsilon_{j}\sigma\big(\tilde{\bv}^\top\tilde{f}(\bx_j)\big)\Big|\bigg).
  \end{align*}
  It then follows that
  \begin{align*}
     & \ebb_{\bm{\epsilon}\sim\{\pm1\}^{n}}\tilde{\tau}\bigg(\sup_{\tilde{f}\in\vcal_{L-1}}\Big\|\sum_{j\in[n]}\epsilon_{j}\tilde{f}(\bx_j)\Big\|_2\bigg)\\
     & \leq \ebb_{\bm{\epsilon}\sim\{\pm1\}^{n}}\tilde{\tau}\bigg(B_{L-1}\sup_{\|\tilde{\bv}\|_2\leq 1,\tilde{f}\in\vcal_{L-2}}\sum_{j\in[n]}\epsilon_{j}\sigma\big(\tilde{\bv}^\top\tilde{f}(\bx_j)\big)\bigg)
     + \ebb_{\bm{\epsilon}\sim\{\pm1\}^{n}}\tilde{\tau}\bigg(B_{L-1}\sup_{\|\tilde{\bv}\|_2\leq 1,\tilde{f}\in\vcal_{L-2}}-\sum_{j\in[n]}\epsilon_{j}\sigma\big(\tilde{\bv}^\top\tilde{f}(\bx_j)\big)\bigg)\\
     & = 2\ebb_{\bm{\epsilon}\sim\{\pm1\}^{n}}\tilde{\tau}\bigg(B_{L-1}\sup_{\|\tilde{\bv}\|_2\leq 1,\tilde{f}\in\vcal_{L-2}}\sum_{j\in[n]}\epsilon_{j}\sigma\big(\tilde{\bv}^\top\tilde{f}(\bx_j)\big)\bigg)
     \leq 2\ebb_{\bm{\epsilon}\sim\{\pm1\}^{n}}\tilde{\tau}\bigg(B_{L-1}\sup_{\|\tilde{\bv}\|_2\leq 1,\tilde{f}\in\vcal_{L-2}}\sum_{j\in[n]}\epsilon_{j}\tilde{\bv}^\top\tilde{f}(\bx_j)\bigg)\\
     & = 2\ebb_{\bm{\epsilon}\sim\{\pm1\}^{n}}\tilde{\tau}\bigg(B_{L-1}\sup_{\|\tilde{\bv}\|_2\leq 1,\tilde{f}\in\vcal_{L-2}}\tilde{\bv}^\top\sum_{j\in[n]}\epsilon_{j}\tilde{f}(\bx_j)\bigg)
     \leq 2\ebb_{\bm{\epsilon}\sim\{\pm1\}^{n}}\tilde{\tau}\bigg(B_{L-1}\sup_{\tilde{f}\in\vcal_{L-2}}\Big\|\sum_{j\in[n]}\epsilon_{j}\tilde{f}(\bx_j)\Big\|_2\bigg).
  \end{align*}
  We can apply the above inequality recursively and derive
  \[
  \ebb_{\bm{\epsilon}\sim\{\pm1\}^{n}}\tilde{\tau}\bigg(\sup_{\tilde{f}\in\vcal_{L-1}}\Big\|\sum_{j\in[n]}\epsilon_{j}\tilde{f}(\bx_j)\Big\|_2\bigg)
  \leq 2^{L-1}\ebb_{\bm{\epsilon}\sim\{\pm1\}^{n}}\tilde{\tau}\bigg(B_{L-1}\cdots B_1\Big\|\sum_{j\in[n]}\epsilon_{j}\bx_j\Big\|_2\bigg).
  \]
  Furthermore, by Eq. \eqref{rade-dnn-1} we know
  \begin{align*}
    \ebb_{\bm{\epsilon}} \sup_{f\in\fcal}\sum_{t\in[d]}\sum_{j\in[n]}\epsilon_{t,j}f_t(\bx_j)
    & = \tau^{-1}\tau\bigg( \ebb_{\bm{\epsilon}} \sup_{f\in\fcal}\sum_{t\in[d]}\sum_{j\in[n]}\epsilon_{t,j}f_t(\bx_j)\bigg)\\
    & \leq \tau^{-1}\bigg(\ebb_{\bm{\epsilon}\sim\{\pm1\}^{n}}\tilde{\tau}\bigg(\sup_{\tilde{f}\in\vcal_{L-1}}\Big\|\sum_{j\in[n]}\epsilon_{j}\tilde{f}(\bx_j)\Big\|_2\bigg)\bigg)\\
    & \leq \tau^{-1}\bigg(2^{L-1}\ebb_{\bm{\epsilon}\sim\{\pm1\}^{n}}\tilde{\tau}\bigg(B_{L-1}\cdots B_1\Big\|\sum_{j\in[n]}\epsilon_{j}\bx_j\Big\|_2\bigg)\bigg)\\
    & = \tau^{-1}\bigg(2^{L-1}\ebb_{\bm{\epsilon}\sim\{\pm1\}^{n}}\exp\bigg(d\lambda\Lambda^2B_L^2B^2_{L-1}\cdots B_1^2\Big\|\sum_{j\in[n]}\epsilon_{j}\bx_j\Big\|_2^2\bigg)\bigg),
  \end{align*}
  where the last identity follows from the definition of $\tilde{\tau}$.
  Let $\lambda_0=d\lambda\Lambda^2B_L^2B^2_{L-1}\cdots B_1^2$. Then
  \begin{align*}
    \ebb_{\bm{\epsilon}\sim\{\pm1\}^{n}}\exp\Big(\lambda_0\Big\|\sum_{j\in[n]}\epsilon_{j}\bx_j\Big\|_2^2\Big) & =
    \ebb_{\bm{\epsilon}\sim\{\pm1\}^{n}}\exp\Big(\lambda_0\sum_{j\in[n]}\|\bx_j\|_2^2+2\lambda_0\sum_{1\leq i<j\leq n}\epsilon_i\epsilon_j\bx_i^\top\bx_j\Big) \\
     & \leq \exp\Big(\lambda_0\sum_{j\in[n]}\|\bx_j\|_2^2\Big)\ebb_{\bm{\epsilon}\sim\{\pm1\}^{n}}\exp\Big(2\lambda_0\sum_{1\leq i<j\leq n}\epsilon_i\epsilon_j\bx_i^\top\bx_j\Big).
  \end{align*}
  We choose $\lambda=\frac{1}{8esd\Lambda^2B_L^2B^2_{L-1}\cdots B_1^2}$, where
  $
  s =\big(\sum_{1\leq i<j\leq n}(\bx_i^\top\bx_j)^2\big)^{\frac{1}{2}}.
  $
  Then it is clear
  $\lambda_0=\frac{1}{8es}$. We can apply Lemma \ref{lem:chaos} to derive that
  \[
     \ebb_{\bm{\epsilon}\sim\{\pm1\}^{n}}\exp\Big(2\lambda_0\sum_{1\leq i<j\leq n}\epsilon_i\epsilon_j\bx_i^\top\bx_j\Big) \leq 2
   \]
   and therefore
   \[
   \ebb_{\bm{\epsilon}\sim\{\pm1\}^{n}}\exp\Big(\lambda_0\Big\|\sum_{j\in[n]}\epsilon_{j}\bx_j\Big\|_2^2\Big) \leq 2\exp\Big(\lambda_0\sum_{j\in[n]}\|\bx_j\|_2^2\Big).
   \]
   We know $\tau^{-1}(x)=\sqrt{\lambda^{-1}\log x}$. It then follows that
   \begin{align*}
     \ebb_{\bm{\epsilon}} \sup_{f\in\fcal}\sum_{t\in[d]}\sum_{j\in[n]}\epsilon_{t,j}f_t(\bx_j) &
     \leq \bigg(\lambda^{-1}(L-1)\log2+\lambda^{-1}\log\ebb_{\bm{\epsilon}\sim\{\pm1\}^{n}}\exp\Big(\lambda_0\Big\|\sum_{j\in[n]}\epsilon_{j}\bx_j\Big\|_2^2\Big)\bigg)^{\frac{1}{2}} \\
      & \leq \bigg(\lambda^{-1}(L-1)\log2+\lambda^{-1}\log\Big(2\exp\Big(\lambda_0\sum_{j\in[n]}\|\bx_j\|_2^2\Big)\Big)\bigg)^{\frac{1}{2}}\\
      & = \bigg(\lambda^{-1}L\log2+\lambda^{-1}\lambda_0\sum_{j\in[n]}\|\bx_j\|_2^2\bigg)^{\frac{1}{2}}\\
      & = \bigg(8esd\Lambda^2B_L^2B^2_{L-1}\cdots B_1^2L\log2+d\Lambda^2B_L^2B^2_{L-1}\cdots B_1^2\sum_{j\in[n]}\|\bx_j\|_2^2\bigg)^{\frac{1}{2}}\\
      & = \sqrt{d}\Lambda B_LB_{L-1}\cdots B_1\bigg(8esL\log2+\sum_{j\in[n]}\|\bx_j\|_2^2\bigg)^{\frac{1}{2}}.
   \end{align*}
   The proof is completed by noting $8e(\log2)\leq16$.
\end{proof}

\begin{remark}\label{rem:golowich}
  Our proof of Proposition \ref{prop:rade-dnn} is motivated by the arguments in \citet{golowich2018size}, which studies Rademacher complexity bounds for DNNs with $d=1$. Our analysis requires to introduce techniques to handle the difficulty of considering $d$ features simultaneously. Indeed, we control the Rademacher complexity for learning with $d$ features by
  \[
  \ebb_{\bm{\epsilon}} \sup_{f\in\fcal}\sum_{t\in[d]}\sum_{j\in[n]}\epsilon_{t,j}f_t(\bx_j)
  \leq \ebb_{\bm{\epsilon}}\bigg(\sup_{\tilde{f}\in\vcal_{L-1},\tilde{\bv}:\|\tilde{\bv}\|_2\leq 1}\sum_{t\in[d]}\Big|\sum_{j\in[n]}\epsilon_{t,j}\sigma(\tilde{\bv}^\top \tilde{f}(\bx_j))\Big|^2\bigg)^{\frac{1}{2}}.
  \]
  If $d=1$, this becomes
  \[
  \ebb_{\bm{\epsilon}} \sup_{f\in\fcal}\sum_{t\in[d]}\sum_{j\in[n]}\epsilon_{t,j}f_t(\bx_j)
  \leq \ebb_{\bm{\epsilon}}\sup_{\tilde{f}\in\vcal_{L-1},\tilde{\bv}:\|\tilde{\bv}\|_2\leq 1}\Big|\sum_{j\in[n]}\epsilon_{j}\sigma(\tilde{\bv}^\top \tilde{f}(\bx_j))\Big|,
  \]
  and the arguments in \citet{golowich2018size} apply. There are two difficulties in applying the arguments in \citet{golowich2018size} to handle general $d\in\nbb$. First, the term $\sum_{t\in[d]}\Big|\sum_{j\in[n]}\epsilon_{t,j}\sigma(\tilde{\bv}^\top \tilde{f}(\bx_j))\Big|$ cannot be written as a Rademacher complexity due to the summation over $t\in[d]$. Second, there is a square function of the term $\Big|\sum_{j\in[n]}\epsilon_{t,j}\sigma(\tilde{\bv}^\top \tilde{f}(\bx_j))\Big|$. To handle this difficulty, we introduce the function $\tau(x)=\exp(\lambda x^2)$ instead of the function $\tau(x)=\exp(\lambda x)$ in \citet{golowich2018size}. To this aim, we need to handle the moment generation function of a Rademacher chaos variable $\sum_{1\leq i<j\leq j}\epsilon_i\epsilon_j(\bx_i^\top\bx_j)^2$, which is not a sub-Gaussian variable. As a comparison, the analysis in \citet{golowich2018size} considers the moment generation function for a sub-Gaussian variable. One can also use the following inequality
  \[
  \bigg(\sum_{t\in[d]}\Big|\sum_{j\in[n]}\epsilon_{t,j}\sigma(\tilde{\bv}^\top \tilde{f}(\bx_j))\Big|^2\bigg)^{\frac{1}{2}}
  \leq \sum_{t\in[d]}\Big|\sum_{j\in[n]}\epsilon_{t,j}\sigma(\tilde{\bv}^\top \tilde{f}(\bx_j))\Big|,
  \]
  the latter of which can then be further controlled by the arguments in \citet{golowich2018size}. This, however, incurs a bound with a linear dependency on $d$. As a comparison, our analysis gives a bound with a square-root dependency on $d$.
\end{remark}

\section{A General Vector-contraction Inequality for Rademacher Complexities\label{sec:gen-cont}}
In this section, we provide a general vector-contraction inequality for Rademacher complexities, which recovers Lemma \ref{lem:maurer} with $\tau(a)=a$. The lemma is motivated from Lemma \ref{lem:rade-cont} by considering a general convex and nondecreasing $\tau$.
\begin{theorem}\label{lem:struct-ext}
  Let $\fcal$ be a class of bounded functions $f:\zcal\mapsto\rbb^d$ which contains the zero function.
  Let $\tau:\rbb_+\to\rbb_+$ be a continuous, non-decreasing and convex function.
  Assume $\tilde{g}_1,\ldots,\tilde{g}_n:\rbb^d\to\rbb$ are $G$-Lipschitz continuous w.r.t. $\|\cdot\|_2$ and satisfy $\tilde{g}_i(\bm{0})=0$.
  Then
  \begin{equation}\label{struct-ext}
    \ebb_{\bm{\epsilon}\sim\{\pm1\}^n}\tau\Big(\sup_{f\in\fcal}\sum_{i=1}^{n}\epsilon_i\tilde{g}_i(f(\bx_i))\Big)\leq
    \ebb_{\bm{\epsilon}\sim\{\pm1\}^{nd}}\tau\Big(G\sqrt{2}\sup_{f\in\fcal}\sum_{i=1}^{n}\sum_{j=1}^{d}\epsilon_{i,j}f_j(\bx_i)\Big).
  \end{equation}
\end{theorem}

The following lemma is due to \citep{maurer2016vector}. We provide here the proof for completeness.
\begin{lemma}\label{lem:structural-rademacher}
  Let $\fcal$ be a class of functions $f:\zcal\mapsto\rbb^d$ and $g$ be any functional defined on $\fcal$.
  Assume that $\tilde{g}_1,\ldots,\tilde{g}_n:\rbb^d\to\rbb$ are $G$-Lipschitz continuous w.r.t. $\|\cdot\|_2$.
  Then,
  \begin{equation}\label{structural-rademacher}
    \ebb_{\bm{\epsilon}\sim\{\pm1\}^n}\sup_{f\in\fcal}\big[g(f)+\sum_{i=1}^{n}\epsilon_i\tilde{g}_i(f(\bx_i))\big] \leq \ebb_{\bm{\epsilon}\sim\{\pm1\}^{nd}}\sup_{f\in\fcal}\Big[g(f)+G\sqrt{2}\sum_{i=1}^{n}\sum_{j=1}^{d}\epsilon_{i,j}f_j(\bx_i)\Big].
  \end{equation}
\end{lemma}
\begin{proof}
   We prove this result by induction.
   According to the symmetry between $f$ and $\tilde{f}$, we derive
  \begin{align}
    & \ebb_{\epsilon_n}\sup_{f\in\fcal}\big[g(f)+\sum_{i=1}^{n}\epsilon_i\tilde{g}_i(f(\bx_i))\big] \notag\\
    & = \frac{1}{2}\sup_{f,\tilde{f}\in \fcal}\Big[g(f)+g(\tilde{f})+\sum_{i=1}^{n-1}\epsilon_i\tilde{g}_i(f(\bx_i))+\sum_{i=1}^{n-1}\epsilon_i\tilde{g}_i(\tilde{f}(\bx_i))+\tilde{g}_n(f(\bx_n))-\tilde{g}_n(\tilde{f}(\bx_n))\Big] \notag\\
    & = \frac{1}{2}\sup_{f,\tilde{f}\in \fcal}\Big[g(f)+g(\tilde{f})+\sum_{i=1}^{n-1}\epsilon_i\tilde{g}_i(f(\bx_i))+\sum_{i=1}^{n-1}\epsilon_i\tilde{g}_i(\tilde{f}(\bx_i))+\big|\tilde{g}_n(f(\bx_n))-\tilde{g}_n(\tilde{f}(\bx_n))\big|\Big],\label{structural-rademacher-1}
  \end{align}
  According to the Lipschitz property and Eq. \eqref{kkahane},
  we derive
  \begin{align*}
    \big|\tilde{g}_n(f(\bx_n))-\tilde{g}_n(\tilde{f}(\bx_n))\big| & \leq  G\big\|f(\bx_n)-\tilde{f}(\bx_n)\big\|_2
     \leq  G\sqrt{2}\ebb_{\epsilon_{n,1},\ldots,\epsilon_{n,j}}\big|\sum_{j=1}^{d}\epsilon_{n,j}\big[f_j(\bx_n)-\tilde{f}_j(\bx_n)\big]\big|.
  \end{align*}
  Plugging the above inequality back into \eqref{structural-rademacher-1} and using the Jensen's inequality, we get
  \begin{align}
    & \ebb_{\epsilon_n}\sup_{f\in\fcal}\big[g(f)+\sum_{i=1}^{n}\epsilon_i\tilde{g}_i(f(\bx_i))\big] \notag\\
    & \leq \frac{1}{2}\ebb_{\epsilon_{n,1},\ldots,\epsilon_{n,j}}\sup_{f,\tilde{f}\in \fcal}\Big[g(f)+g(\tilde{f})+\sum_{i=1}^{n-1}\epsilon_i\tilde{g}_i(f(\bx_i))+\sum_{i=1}^{n-1}\epsilon_i\tilde{g}_i(\tilde{f}(\bx_i))+G\sqrt{2}\big|\sum_{j=1}^{d}\epsilon_{n,j}\big[f_j(\bx_n)-\tilde{f}_j(\bx_n)\big]\big|\Big] \notag\\
    & = \frac{1}{2}\ebb_{\epsilon_{n,1},\ldots,\epsilon_{n,j}}\sup_{f,\tilde{f}\in \fcal}\Big[g(f)+g(\tilde{f})+\sum_{i=1}^{n-1}\epsilon_i\tilde{g}_i(f(\bx_i))+\sum_{i=1}^{n-1}\epsilon_i\tilde{g}_i(\tilde{f}(\bx_i))+G\sqrt{2}\sum_{j=1}^{d}\epsilon_{n,j}\big[f_j(\bx_n)-\tilde{f}_j(\bx_n)\big]\Big]\notag\\
    & = \ebb_{\epsilon_{n,1},\ldots,\epsilon_{n,j}}\sup_{f\in\fcal}\Big[g(f)+\sum_{i=1}^{n-1}\epsilon_i\tilde{g}_i(f(\bx_i))+G\sqrt{2}\sum_{j=1}^{d}\epsilon_{n,j}f_j(\bx_n)\Big],\notag
  \end{align}
  where we have used the symmetry in the second step.

  The stated result can be derived by continuing the above deduction with expectation over $\epsilon_{n-1}$, $\epsilon_{n-2}$ and so on.
\end{proof}

To prove Theorem \ref{lem:struct-ext}, we introduce the following lemmas on the approximation of a continuous, non-decreasing and convex function. Let $a_+=\max\{a,0\}$.
\begin{lemma}\label{lem:convex-representation}
  Let $f:[a,b]\to\rbb_+$ be a continuous, non-decreasing and convex function and $m\geq2$. Let $a=x_1<\cdots<x_m=b$. Then the function $\tilde{g}:[a,b]\to\rbb$ defined by
  $$
    \tilde{g}(x)=f(x_k)+\frac{f(x_{k+1})-f(x_k)}{x_{k+1}-x_k}(x-x_k),\quad\text{if }x\in[x_k,x_{k+1}]
  $$
  belongs to the set
  \begin{equation}\label{Lambda}
  H_{[a,b]}:=\Big\{c_0+\sum_{i=1}^{m}c_i(x-t_i)_+:c_i\geq0,i\in[n],t_i\in\rbb,m\in\nbb,x\in[a,b]\Big\}.
  \end{equation}
\end{lemma}
\begin{proof}
  Define
  $$
    \bar{f}(x)=f(x_1)+\sum_{i=1}^{m-1}\frac{f(x_{i+1})-f(x_i)}{x_{i+1}-x_i}\big[(x-x_i)_+-(x-x_{i+1})_+\big].
  $$
  We first show that $\bar{f}(x)=\tilde{g}(x)$ for all $x\in[a,b]$. Suppose that $x\in[x_k,x_{k+1})$. Then, it is clear that
  \begin{align*}
    \bar{f}(x) & =f(x_1)+\sum_{i=1}^{k-1}\frac{f(x_{i+1})-f(x_i)}{x_{i+1}-x_i}\big[(x-x_i)_+-(x-x_{i+1})_+\big]+\frac{f(x_{k+1})-f(x_k)}{x_{k+1}-x_k}\big[(x-x_k)_+-(x-x_{k+1})_+\big]\\
     & \qquad +\sum_{i=k+1}^{m-1}\frac{f(x_{i+1})-f(x_i)}{x_{i+1}-x_i}\big[(x-x_i)_+-(x-x_{i+1})_+\big]\\
     & = f(x_1)+\sum_{i=1}^{k-1}\frac{f(x_{i+1})-f(x_i)}{x_{i+1}-x_i}\big[(x-x_i)-(x-x_{i+1})\big]+\frac{f(x_{k+1})-f(x_k)}{x_{k+1}-x_k}\big[(x-x_k)-0\big]\\
     & = f(x_k) + \frac{f(x_{k+1})-f(x_k)}{x_{k+1}-x_k}(x-x_k)=\tilde{g}(x).
  \end{align*}

  We now show that $\bar{f}(x)$ belongs to the set $H_{[a,b]}$. Indeed, it follows from $(x-x_m)_+=0$ for all $x\leq x_m=b$ that
  \begin{align*}
    \bar{f}(x) & = f(x_1) + \sum_{i=1}^{m-1}\frac{f(x_{i+1})-f(x_i)}{x_{i+1}-x_i}(x-x_i)_+-\sum_{i=2}^m\frac{f(x_i)-f(x_{i-1})}{x_i-x_{i-1}}(x-x_i)_+ \\
     & = f(x_1) + \frac{f(x_2)-f(x_1)}{x_2-x_1}(x-x_1)_++\sum_{i=2}^{m-1}\Big[\frac{f(x_{i+1})-f(x_i)}{x_{i+1}-x_i}-\frac{f(x_i)-f(x_{i-1})}{x_i-x_{i-1}}\Big](x-x_i)_+.
  \end{align*}
  Therefore, $\bar{f}(x)$ can be written as
  $
    \bar{f}(x) = c_0+\sum_{i=1}^{m-1}c_i(x-t_i)_+
  $ with $t_i=x_i,c_0=f(x_1),c_1=\frac{f(x_2)-f(x_1)}{x_2-x_1}$  and $c_i=\frac{f(x_{i+1})-f(x_i)}{x_{i+1}-x_i}-\frac{f(x_i)-f(x_{i-1})}{x_i-x_{i-1}},i=2,\ldots,m-1$.
  The terms $c_1,\ldots,c_{m-1}$ are all non-negative since $f$ is non-decreasing and convex. The proof is completed.
\end{proof}

\begin{lemma}\label{lem:convex-hull}
  If $f:[a,b]\to\rbb_+$ is continuous, non-decreasing and convex, then $f$ belongs to the closure of $H_{[a,b]}$ defined in Eq. \eqref{Lambda}.
\end{lemma}
\begin{proof}
  Let $m\in\nbb$. We can find $a=x_1^{(m)}<x_2^{(m)}<\cdots<x_{n+1}^{(m)}=b$ such that
  $$
    f(x_k^{(m)})-f(x_{k-1}^{(m)})\leq\frac{f(b)-f(a)}{n}.
  $$
  Introduce
  $$
    f^{(m)}(x):=f(x_k^{(m)})+\frac{f(x_{k+1}^{(m)})-f(x_k^{(m)})}{x_{k+1}^{(m)}-x_k^{(m)}}(x-x_k^{(m)})\quad\text{if }x\in[x_k^{(m)},x_{k+1}^{(m)}].
  $$
  For any $x\in[x_k^{(m)},x_{k+1}^{(m)}]$, it follows from the convexity of $f$ that
  \begin{align*}
    |f^{(m)}(x)-f(x)| & = \Big|f(x_k^{(m)})-f(x)+\frac{\big(f(x_{k+1}^{(m)})-f(x_k^{(m)})\big)(x-x_k^{(m)})}{x_{k+1}^{(m)}-x_k^{(m)}}\Big| \\
     & = f(x_k^{(m)}) - f(x) + \frac{\big(f(x_{k+1}^{(m)})-f(x_k^{(m)})\big)(x-x_k^{(m)})}{x_{k+1}^{(m)}-x_k^{(m)}} \\
     & \leq \frac{\big(f(x_{k+1}^{(m)})-f(x_k^{(m)})\big)(x-x_k^{(m)})}{x_{k+1}^{(m)}-x_k^{(m)}}\leq \frac{f(b)-f(a)}{n},
  \end{align*}
  from which we know $\lim_{n\to\infty}|f^{(m)}(x)-f(x)|=0$ for all $x\in[a,b]$. Lemma \ref{lem:convex-representation} shows that $f^{(m)}\in H_{[a,b]}$ for all $m\in\nbb$.
  Therefore, $f$ belongs to the closure of $H_{[a,b]}$.
  The proof is completed.
\end{proof}

\begin{proof}[Proof of Theorem \ref{lem:struct-ext}]
According to the boundedness assumption of $f\in\fcal$ and the fact $\bm{0}\in \fcal$, there exist $B>0$ such that
\begin{multline*}
  0\leq \min\Big\{\sup_{f\in\fcal}\sum_{i=1}^{n}\epsilon_i\tilde{g}_i(f(\bx_i)),G\sqrt{2}\sup_{f\in\fcal}\sum_{i=1}^{n}\sum_{j=1}^{d}\epsilon_{i,j}f_j(\bx_i)\Big\}\\
  \leq\max\Big\{\sup_{f\in\fcal}\sum_{i=1}^{n}\epsilon_i\tilde{g}_i(f(\bx_i)),G\sqrt{2}\sup_{f\in\fcal}\sum_{i=1}^{n}\sum_{j=1}^{d}\epsilon_{i,j}f_j(\bx_i)\Big\}\leq B
\end{multline*}
for all $\bm{\epsilon}\in\{\pm1\}^n$.
Let $t\in\rbb$ be an arbitrary number. Define $g_t:\fcal\mapsto\rbb$ by $g_t(f)=0$ for any $f\neq \bm{0}$
and $g_t(\bm{0})=t$. It is clear that
$$
  \ebb_{\bm{\epsilon}\sim\{\pm1\}^{n}}\sup_{f\in\fcal}\big[g_t(f)+\sum_{i=1}^{n}\epsilon_i\tilde{g}_i(f(\bx_i))\big] =
  \ebb_{\bm{\epsilon}\sim\{\pm1\}^{n}}\max\Big\{\sup_{f\in\fcal:f\neq\bm{0}}\big[\sum_{i=1}^{n}\epsilon_i\tilde{g}_i(f(\bx_i))\big],t\Big\}
$$
and
$$
  \ebb_{\bm{\epsilon}\sim\{\pm1\}^{nd}}\sup_{f\in\fcal}\big[g_t(f)+G\sqrt{2}\sum_{i=1}^{n}\sum_{j=1}^{d}\epsilon_{i,j}f_j(\bx_i)\big] =
  \ebb_{\bm{\epsilon}\sim\{\pm1\}^{nd}}\max\Big\{G\sqrt{2}\sup_{f\in\fcal:f\neq\bm{0}}\big[\sum_{i=1}^{n}\sum_{j=1}^{d}\epsilon_{i,j}f_j(\bx_i)\big],t\Big\}.
$$
Plugging the above identities into \eqref{structural-rademacher} with $g=g_t$ gives
$$
  \ebb_{\bm{\epsilon}\sim\{\pm1\}^{n}}\max\Big\{\sup_{f\in\fcal:f\neq\bm{0}}\big[\sum_{i=1}^{n}\epsilon_i\tilde{g}_i(f(\bx_i))\big],t\Big\} \leq \ebb_{\bm{\epsilon}\sim\{\pm1\}^{nd}}\max\Big\{G\sqrt{2}\sup_{f\in\fcal:f\neq\bm{0}}\big[\sum_{i=1}^{n}\sum_{j=1}^{d}\epsilon_{i,j}f_j(\bx_i)\big],t\Big\}.
$$
If $t\geq0$, the above inequality is equivalent to
\begin{equation}\label{struct-ext-1}
  \ebb_{\bm{\epsilon}\sim\{\pm1\}^{n}}\max\Big\{\sup_{f\in\fcal}\big[\sum_{i=1}^{n}\epsilon_i\tilde{g}_i(f(\bx_i))\big],t\Big\} \leq \ebb_{\bm{\epsilon}\sim\{\pm1\}^{nd}}\max\Big\{G\sqrt{2}\sup_{f\in\fcal}\big[\sum_{i=1}^{n}\sum_{j=1}^{d}\epsilon_{i,j}f_j(\bx_i)\big],t\Big\}
\end{equation}
by noting $\tilde{g}_i(\bm{0})=0$ for all $i\in\nbb_n$.
If $t<0$, it follows from \eqref{structural-rademacher} with $g(f)=0$ that
\begin{align*}
  & \ebb_{\bm{\epsilon}\sim\{\pm1\}^{n}}\max\Big\{\sup_{f\in\fcal}\big[\sum_{i=1}^{n}\epsilon_i\tilde{g}_i(f(\bx_i))\big],t\Big\} \\
  & = \ebb_{\bm{\epsilon}\sim\{\pm1\}^{n}}\sup_{f\in\fcal}\big[\sum_{i=1}^{n}\epsilon_i\tilde{g}_i(f(\bx_i))\big]
   \leq \ebb_{\bm{\epsilon}\sim\{\pm1\}^{nd}}\sup_{f\in\fcal}\Big[G\sqrt{2}\sum_{i=1}^{n}\sum_{j=1}^{d}\epsilon_{i,j}f_j(\bx_i)\Big] \\
   & = \ebb_{\bm{\epsilon}\sim\{\pm1\}^{nd}}\max\Big\{G\sqrt{2}\sup_{f\in\fcal}\big[\sum_{i=1}^{n}\sum_{j=1}^{d}\epsilon_{i,j}f_j(\bx_i)\big],t\Big\},
\end{align*}
where we have used $\tilde{g}_i(\bm{0})=0$ for all $i\in\nbb_n$ in the first identity.
That is, \eqref{struct-ext-1} holds for all $t\in\rbb$. Subtracting $t$ from both sides of Eq. \eqref{struct-ext-1} gives
\begin{equation}\label{struct-ext-2}
  \ebb_{\bm{\epsilon}\sim\{\pm1\}^{n}}\Big(\sup_{f\in\fcal}\sum_{i=1}^{n}\epsilon_i\tilde{g}_i(f(\bx_i))-t\Big)_+ \leq \ebb_{\bm{\epsilon}\sim\{\pm1\}^{nd}}\Big(G\sqrt{2}\sup_{f\in\fcal}\big[\sum_{i=1}^{n}\sum_{j=1}^{d}\epsilon_{i,j}f_j(\bx_i)\big]-t\Big)_+,
  \quad\forall t\in\rbb,
\end{equation}
from which we know
$$
  \ebb_{\bm{\epsilon}\sim\{\pm1\}^{n}}\tilde{\tau}\Big(\sup_{f\in\fcal}\sum_{i=1}^{n}\epsilon_i\tilde{g}_i(f(\bx_i))\Big)\leq
    \ebb_{\bm{\epsilon}\sim\{\pm1\}^{nd}}\tilde{\tau}\Big(G\sqrt{2}\sup_{f\in\fcal}\sum_{i=1}^{n}\sum_{j=1}^{d}\epsilon_{i,j}f_j(\bx_i)\Big),\quad\forall\tilde{\tau}\in H_{[0,B]}.
$$
According to Lemma \ref{lem:convex-hull}, we know $\tau:[0,B]\to\rbb_+$ belongs to the closure of $H_{[0,B]}$. Therefore, Eq. \eqref{struct-ext} holds.
The proof is completed.
\end{proof}

\section{Lipschitz Continuity of Loss Functions}
The following proposition is known in the literature~\citep{lei2019data}. We prove it for completeness.
\begin{proposition}\label{prop:loss-lip}
\begin{enumerate}
  \item[(a)] Let $\ell$ be defined as Eq. \eqref{loss-a}. Then $\ell$ is $1$-Lipschitz continuous w.r.t. $\|\cdot\|_\infty$ and $1$-Lipschitz continuous w.r.t. $\|\cdot\|_2$.
  \item[(b)] Let $\ell$ be defined as Eq. \eqref{loss-b}. Then $\ell$ is $1$-Lipschitz continuous w.r.t. $\|\cdot\|_\infty$ and $1$-Lipschitz continuous w.r.t. $\|\cdot\|_2$.
\end{enumerate}
\end{proposition}
\begin{proof}
  We first prove Part (a). For any $\bv$ and $\bv'$, we have
  \begin{align*}
  |\ell(\bv)-\ell(\bv')|&=\big|\max\big\{0,1+\max_{i\in[k]}\{-v_i\}\big\}-\max\big\{0,1+\max_{i\in[k]}\{-v'_i\}\big\}\big|\\
  &\leq |\max_{i\in[k]}\{-v_i\}-\max_{i\in[k]}\{-v'_i\}|\leq \max_{i\in[k]}|v_i-v'_i|=\|\bv-\bv'\|_\infty,
  \end{align*}
  where we have used the elementary inequality
  \[
  |\max_{i\in[k]}a_i-\max_{i\in[k]}b_i|\leq \max_{i\in[k]}|a_i-b_i|.
  \]
  This proves Part (a).

  We now prove Part (b). It is clear that
  \[
  \frac{\partial\ell(\bv)}{\partial v_i}=\frac{-\exp(-v_i)}{1+\sum_{i\in[k]}\exp(-v_i)}.
  \]
  Therefore, the $\ell_1$ norm of the gradient can be bounded as follows
  \[
  \|\nabla \ell(\bv)\|_1\leq\frac{1}{1+\sum_{i\in[k]}\exp(-v_i)}\sum_{i\in[k]}\exp(-v_i)\leq1.
  \]
  This proves Part (b). The proof is completed.
\end{proof}

\end{document}